\documentclass[11pt,a4paper]{article}

\usepackage{fullpage}		 
\usepackage[utf8]{inputenc} 
\usepackage[T1]{fontenc}    
\usepackage{lmodern}			 
\usepackage{hyperref}       
\usepackage{url}            
\usepackage{booktabs}       
\usepackage{amsthm}			 
\usepackage{amsmath}			 
\usepackage{amssymb}			 
\usepackage{amsfonts}       
\usepackage{mathtools}      
\usepackage{nicefrac}       
\usepackage{microtype}      
\usepackage{subfiles}       
\usepackage{xcolor}         
\usepackage{bm}             
\usepackage{enumitem}		 
\usepackage[numbers]{natbib}  
\usepackage{siunitx}

\usepackage{xpatch}
\makeatletter
\xpatchcmd{\@thm}{\thm@headpunct{.}}{\thm@headpunct{}}{}{}
\makeatother

\makeatletter
\def\th@plain{%
  \thm@notefont{}
  \itshape 
}
\def\th@definition{%
  \thm@notefont{}
  \normalfont 
}
\makeatother

\newtheorem{theorem}{Theorem}
\newtheorem{lemma}[theorem]{Lemma}

\newtheorem{proposition}[theorem]{Proposition}

\newtheorem{definition}{Definition}

\newtheorem{corollary}[theorem]{Corollary}

\definecolor{myred}{HTML}{880000}
\definecolor{mygreen}{HTML}{008800}
\definecolor{myblue}{HTML}{000088}
\definecolor{linkblue}{HTML}{0000BB}

\hypersetup{
  colorlinks,
  linkcolor={myred},
  citecolor={myblue},
  urlcolor={linkblue}
}

\newcommand{\R}{\mathbb{R}}
\newcommand{\E}{\mathbb{E}}
\newcommand{\x}{\mathbf{x}}
\newcommand{\X}{\mathbf{X}}
\newcommand{\Y}{\mathbf{Y}}
\newcommand{\A}{\mathbf{A}}
\newcommand{\U}{\mathbf{U}}
\newcommand{\V}{\mathbf{V}}

\newcommand{\gauss}{\mathcal{N}}
\newcommand{\tr}{\operatorname{tr}}
\renewcommand{\R}{\mathbb{R}}
\renewcommand{\O}{\mathcal{O}}

\renewcommand{\S}{\mathbb{S}}

\title{Implicit Regularization in Matrix Sensing via Mirror Descent}
\author{
  Fan Wu,
  Patrick Rebeschini \\
	Department of Statistics, University of Oxford
}

\begin{document}

\maketitle

\begin{abstract}%
	We study discrete-time mirror descent applied to the unregularized empirical risk in matrix sensing. In both the general case of rectangular matrices and the particular case of positive semidefinite matrices, a simple potential-based analysis in terms of the Bregman divergence allows us to establish convergence of mirror descent---with different choices of the mirror maps---to a matrix that, among all global minimizers of the empirical risk, minimizes a quantity explicitly related to the nuclear norm, the Frobenius norm, and the von Neumann entropy. In both cases, this characterization implies that mirror descent, a first-order algorithm minimizing the unregularized empirical risk, recovers low-rank matrices under the same set of assumptions that are sufficient to guarantee recovery for nuclear-norm minimization. When the sensing matrices are symmetric and commute, we show that gradient descent with full-rank factorized parametrization is a first-order approximation to mirror descent, in which case we obtain an explicit characterization of the implicit bias of gradient flow as a by-product.

\end{abstract}

\section{Introduction}
\label{sec:introduction}
Matrix sensing represents a paradigm in modern statistics \cite{CLC19, R11, RFP10}, with applications ranging from image compression \cite{AP76} to collaborative filtering \cite{KBV09} and dimensionality reduction \cite{WS06}, for instance. The goal is to recover a rank-$r$ matrix $\X^\star\in\R^{n\times n'}$ from a set of linear measurements $y_i = \langle\A_i, \X^\star\rangle$, $i=1,\dots,m$, where the sensing matrices $\A_i\in\R^{n\times n'}$ are observed. This formulation includes the problem of matrix completion, where a subset of the entries of the matrix $\X^\star$ is observed.

Most of the literature on matrix sensing is based on some form of explicit regularization or rank constraint to encourage or enforce low-rankness of the estimated matrix. A popular approach is based on minimizing the nuclear norm or on using explicit regularization techniques based on the nuclear norm, e.g.\ \cite{CCS10, JMD10, KMO10, MGC11, R11, RFP10, TY10}. Another popular approach is based on non-convex optimization and the low-rank factorization $\X = \U\V^\top$ with matrices $\U\in\R^{n\times r}$, $\V\in\R^{n'\times r}$, where the factorization itself enforces the low-rankness of $\X$, e.g.\ \cite{CW15, CLC19, JKN16, MLC21, MWCC18, SL16, TBSSR16, ZL16}.  

The literature on \emph{implicit} regularization for matrix sensing is more recent and less well developed. When $\X^\star$ is assumed to be a positive semidefinite matrix, it was first empirically observed in \cite{GWBNS17} that mininizing the unregularized empirical risk using vanilla gradient descent with parametrization $\X = \U\U^\top$ and random full-rank initialization close to zero yields a low nuclear norm solution, even when $\U\in\R^{n\times n}$, i.e.\ when no constraint on the rank of $\X$ is enforced. It was later proved that, with this parametrization, gradient descent minimizes the nuclear norm under the assumption that the sensing matrices $\A_i$'s commute \cite{GWBNS17}, or when the sensing matrices satisfy a restricted isometry property \cite{LMZ18}. Gradient descent with low-rank initialization within a specific ``capture neighborhood'' has been studied in \cite{EZ21}, which ensures that the iterates of the algorithm stay low-rank. When $\X^\star$ is a general rectangular matrix, implicit regularization in matrix sensing has been studied through the lenses of deep matrix factorization in \cite{ACHL19, GSD20, LLL21, RC20}, and empirical and theoretical evidence is provided which suggests that a notion of rank-minimization is involved in the implicit bias of gradient descent. However, these works do not establish an explicit characterization of the limiting point of the optimization algorithm, e.g.\ in the form of a quantity that is minimized among all minimizers of the empirical risk. In the special case of full-observation matrix sensing, gradient flow has been shown to learn solutions with gradually increasing rank \cite{GBL19}.

Most theoretical results on implicit regularization in matrix sensing consider the continuous-time dynamics (i.e.\ gradient flow), e.g.\ \cite{ACHL19, EZ21, GSD20, GWBNS17, LLL21, RC20}, or assume an infinitesimally small initialization, e.g.\ \cite{GWBNS17, LLL21}. Notable exceptions are \cite{GBL19}, which makes a commutativity assumption on the data matrices in discrete time, and \cite{LMZ18}, which assumes a restricted isometry property that leads to a suboptimal sample complexity when the sensing matrices belong to a general class of random matrices \cite{RFP10}.
In the context of linear neural networks, the dependence of the implicit bias of gradient descent on the initialization has been studied in \cite{AMN+21, MWG+20, WGL+20, YKM21}. For instance, linear diagonal networks with shared weights were considered in \cite{WGL+20}, where it was shown that gradient descent minimizes a quantity which corresponds to the $\ell_1$-norm in the limit $\alpha\rightarrow 0$ and to the (weighted) $\ell_2$-norm in the limit $\alpha\rightarrow \infty$, where $\alpha$ denotes the initialization size. This result was later generalized in \cite{YKM21} using a tensor formulation, which allows for architectures including linear diagonal networks and linear full-length convolutional networks. However, these results focus on vector-based notions of norm-like functions that do not capture matrix-based quantities typically of interest in matrix sensing. 

\subsection{Our contributions}
We study the implicit bias of discrete-time mirror descent in matrix sensing in both the general case of rectangular matrices and the particular case of positive semidefinite matrices. Under the only assumption on the sensing matrices $\A_i$'s that there exists a matrix achieving zero training error, we characterize the limiting point as the matrix that minimizes a quantity which interpolates between the nuclear norm and Frobenius norm, parametrized by the mirror map parameter, in the rectangular case; and which is a linear combination of the nuclear norm and the negative von Neumann entropy, parametrized by the initialization size, in the positive semidefinite case. Compared to results on implicit regularization for gradient descent, our framework for mirror descent is simple, and the same analysis yields results for both the case of general rectangular and positive semidefinite matrices.

In the general case of rectangular matrices, we show that mirror descent, initialized at zero and equipped with the spectral hypentropy mirror map \cite{GHS20} parametrized by $\beta > 0$, among all global minimizers of the empirical risk, converges to a matrix that minimizes a quantity interpolating between the nuclear norm in the limit $\beta\rightarrow 0$ and the Frobenius norm in the limit $\beta\rightarrow\infty$. As a consequence, our result implies that, for $\beta \rightarrow 0$, mirror descent can recover a rank-$r$ matrix $\X^\star$ under the same set of assumptions that is sufficient for nuclear norm minimization to be successful, namely when the sensing matrices $\A_i$'s satisfy the restricted isometry property with restricted isometry constant smaller than some absolute constant \cite{RFP10}, or if $\X^\star$ satisfies an incoherence condition and $r(n+n')$ (modulo constants and logarithmic term) random entries of $\X^\star$ are observed \cite{R11}. To the best of our knowledge, this is the first recovery guarantee for an implicit regularization-based algorithm that does not explicitly enforce low-rankness in general rectangular matrix sensing.

In the particular case of positive semidefinite matrices, we can alternatively consider the spectral entropy mirror map and show that mirror descent, initialized at $\alpha\mathbf{I}$ for any $\alpha > 0$, converges to a positive semidefinite matrix that minimizes a linear combination of the nuclear norm and the negative von Neumann entropy, where the relative weights are controlled by the initialization size $\alpha$. While the limit $\alpha\rightarrow 0$ corresponds to minimizing the nuclear norm as in the case of rectangular matrices, the limit $\alpha\rightarrow \infty$ corresponds to \emph{maximizing} the nuclear norm. This also translates into guaranteed recovery of a low-rank matrix $\X^\star$ for $\alpha\rightarrow 0$ under the same assumptions that are sufficient for nuclear norm minimization \cite{R11, RFP10}. A comparable result for gradient descent with full-rank factorized parametrization has been established in \cite{LMZ18} under a stronger assumption on the restricted isometry constant, which is assumed to be smaller than a quantity depending on both the rank and the condition number of the matrix $\X^\star$ and translates into a sub-optimal sample complexity.

We establish our results using a potential-based analysis for mirror descent in terms of the Bregman divergence, which provides an alternative proof technique to characterize the limiting point of mirror descent compared to the analysis based on KKT optimality conditions used in \cite{GLSS18}. As a by-product, our proof of Theorem \ref{theorem:bias} yields an alternative proof of Theorem 1 in \cite{GLSS18}. The advantage of our approach is that convergence of mirror descent does not need to be assumed a-priori and can instead be established using convexity of the empirical risk. The analysis in terms of the Bregman divergence is not limited to convex settings and has been applied to non-convex problems, e.g.\ \cite{WR20b, WR21, ZMBBG20}, and hence can be of more general interest to investigate the phenomenon of implicit bias. 

In the case of square matrices, we show, assuming that the sensing matrices $\A_i$'s are symmetric and commute, that gradient descent with full-rank parametrization $\X = \U\U^\top - \V\V^\top$, where $\U,\V\in\R^{n\times n}$, is a first-order approximation to mirror descent equipped with the spectral hypentropy, where the initialization size corresponds to the mirror map parameter $\beta$. A similar connection between mirror descent and reparametrized gradient descent has been established in the vector-case \cite{AW20, GHS20, VKR20, WR20b}. Similarly, we show that gradient descent with full-rank parametrization $\X = \U\U^\top$, $\U\in\R^{n\times n}$, is a first-order approximation to mirror descent equipped with the spectral entropy when the sensing matrices are symmetric and commute, thus recovering a generalization of Theorem 1 in \cite{GWBNS17} which holds for any positive initialization size $\alpha > 0$ (rather than in the limit $\alpha\rightarrow 0$). 

We present numerical simulations which suggest that, in some regimes, our results on the dependence on the initialization size $\alpha$ for mirror descent might be indicative for the behavior of gradient descent, even when the sensing matrices $\A_i$'s do not commute. More precisely, the final estimates of gradient descent and mirror descent closely track each other when the number of measurements is sufficiently large for nuclear norm minimization to recover a planted low-rank matrix, while gradient descent seems to put more emphasis on lowering the effective rank \cite{RV07} at the expense of a higher nuclear norm when fewer measurements are available, which supports the empirical observations in \cite{ACHL19, LLL21}.

\section{Background}
\label{sec:background}
We begin by introducing some notation used throughout this paper. We use boldface uppercase letters to denote matrices and boldface lowercase letters to denote vectors. We write $\|\cdot\|_*$, $\|\cdot\|_F$ and $\|\cdot\|_2$ for the nuclear, Frobenius and spectral norm, respectively, and denote by $\langle \X, \Y \rangle = \tr(\X^\top\Y)$ the standard Frobenius inner product. We write $\S^n,\S^n_+\subseteq\R^{n\times n}$ for the set of symmetric and positive semidefinite matrices, respectively. Without loss of generality, we will always assume $n\le n'$.

We first give a brief overview of unconstrained mirror descent (with matrix arguments). 

Let $\mathcal{D}\subseteq \R^{n\times n'}$ be an open convex set. We say that $\Phi : \mathcal{D} \rightarrow \R$ is a \emph{mirror map} if it is strictly convex, differentiable, and its gradient takes all possible values, i.e.\ $\{\nabla \Phi(\X):\X \in \mathcal{D}\} = \R^{n\times n'}$.

Given a mirror map $\Phi$, the associated \emph{Bregman divergence} is defined as
\begin{equation}
\label{eq:bregman_divergence}
D_\Phi(\X, \Y) = \Phi(\X) - \Phi(\Y) - \langle \nabla\Phi(\Y), \X - \Y \rangle.
\end{equation}
Then, the mirror descent algorithm to minimize a function $f:\mathcal{D}\rightarrow \R$ is defined by the update
\begin{equation}
\label{eq:mirror_descent}
\nabla\Phi(\X_{t+1}) = \nabla\Phi(\X_t) - \eta_t\nabla f(\X_t),
\end{equation}
where $\eta_t> 0$ is a sequence of step sizes. We approach the problem of matrix sensing, where we are given measurements $\{\A_i, y_i\}_{i=1}^m$, by minimizing the unregularized empirical risk
\begin{equation}
  \label{eq:objective}
  f(\X) = \frac{1}{2m}\sum_{i=1}^m\bigl(\langle \A_i, \X\rangle - y_i\bigr)^2
\end{equation}
using mirror descent equipped with the spectral hypentropy mirror map, which is defined as \cite{GHS20}
\begin{equation}
  \label{eq:mirror_map_hypentropy}
  \Phi_\beta(\X) = \sum_{i=1}^{n}\sigma_i \operatorname{arcsinh}\biggl(\frac{\sigma_i}{\beta}\biggr) - \sqrt{\sigma_i^2 + \beta^2},
\end{equation}
for some $\beta > 0$, where $\{\sigma_i\}_{i=1}^n$ denote the singular values of $\X$. We provide expressions for the gradient $\nabla\Phi_\beta$ and discuss the per-iteration computational cost of the corresponding mirror descent algorithm in Appendix \ref{appendix:mirror_maps_gradients}.

If the optimization is restricted to positive semidefinite matrices $\X\in\R^{n\times n}$, we can replace $\R^{n\times n'}$ with $\S^{n}$ in above definitions and consider the spectral entropy (using the convention $0\log 0 = 0$)
\begin{equation}
  \label{eq:mirror_map}
  \Phi(\X) = \tr(\X\log\X - \X),
\end{equation}
which is well-defined on the set of positive semidefinite matrices $\S^n_+$.

\section{Algorithmic regularization of mirror descent}
\label{sec:algorithmic_regularization_of_mirror_descent}
In this section, we show that mirror descent, equipped with the spectral hypentropy mirror map \eqref{eq:mirror_map_hypentropy} in the rectangular case and the spectral entropy mirror map \eqref{eq:mirror_map} in the positive semidefinite case, converges to a global minimizer of the empirical risk $f$ that minimizes a quantity which interpolates between the nuclear norm and the Frobenius norm in the rectangular case, and is a linear combination of the nuclear norm and the negative von Neumann entropy $\tr(\X\log\X) = \sum_{i=1}^n\lambda_i\log\lambda_i$ in the positive semidefinite case, where $\{\lambda_i\}_{i=1}^n$ denote the eigenvalues of the matrix $\X\in\S^n_+$. 

The following Theorem applies to matrix sensing in the case of general rectangular matrices.
\begin{theorem}[Rectangular case]
\label{theorem:bias}
Consider the mirror descent algorithm \eqref{eq:mirror_descent} with mirror map \eqref{eq:mirror_map_hypentropy} with parameter $\beta > 0$ and initialization $\X_0 = \mathbf{0}$. Suppose that the step sizes satisfy $\eta_t \equiv \eta \le c$, where $c>0$ is a constant depending on the sensing matrices $\A_i$'s, observations $y_i$'s and parameter $\beta$, and that there exists a matrix $\X'$ satisfying $f(\X') = 0$. Then, mirror descent converges to a matrix $\X_\infty = \lim_{t\rightarrow \infty} \X_t$ which, among all global minimizers of $f$, minimizes
\begin{equation}
  \label{eq:claim1}
  \sum_{i=1}^{n}\sigma_i\log\frac{1}{\beta} + \sigma_i\log\Bigl(\sigma_i + \sqrt{\sigma_i^2 + \beta^2}\Bigr) - \sqrt{\sigma_i^2 + \beta^2},
\end{equation}
where $\{\sigma_i\}_{i=1}^{n}$ denote the singular values of $\X_\infty$.
We have, for any $t\ge 0$,
\begin{equation}
  \label{eq:claim2}
  f(\X_t) \le \frac{D_{\Phi_\beta}(\X_\infty, \X_0)}{\eta t}.
\end{equation}
\end{theorem}
\begin{proof}[Proof sketch]
  The following identity is key to our proof and characterizes how the Bregman divergence between any reference point $\X'$ and the mirror descent iterates $\X_t$ evolves:
  \begin{equation}
  \label{eq:bregman_difference_sketch}
  D_{\Phi_\beta}(\X', \X_{t+1}) - D_{\Phi_\beta}(\X', \X_t) = - \eta\langle \nabla f(\X_t), \X_t - \X'\rangle + D_{\Phi_\beta}(\X_t, \X_{t+1}),
  \end{equation}
  which follows from the definition of the Bregman divergence \eqref{eq:bregman_divergence} and the mirror descent update \eqref{eq:mirror_descent}.
  Letting $\X'$ be any global minimizer of $f$, we can compute 
  $
  \langle \nabla f(\X_t), \X_t - \X'\rangle = 2f(\X_t),
  $
  where we used the assumption that there exists a matrix $\X'$ achieving zero training error $f(\X') = 0$.
  Using the strong convexity of the spectral hypentropy mirror map, we can bound $D_\Phi(\X_t, \X_{t+1})$ to show
  \begin{equation*}
  D_{\Phi_\beta}(\X', \X_{t+1}) - D_{\Phi_\beta}(\X',\X_t) \le -\eta f(\X_t),
  \end{equation*}
  for any global minimizer $\X'$ of $f$. Since the Bregman divergence $D_{\Phi_\beta}(\X',\X_t)$ is bounded from below by zero, this means that the empirical risk $f(\X_t)$ must converge to zero, which in turn implies that $\X_t$ converges to a global minimizer of $f$. 
  
  To see \emph{which} global minimizer mirror descent converges to, observe that the difference in \eqref{eq:bregman_difference_sketch} does not depend on the reference point $\X'$, as long as $\X'$ is a global minimizer of $f$. This means that the Bregman divergence $D_{\Phi_\beta}(\X', \X_t)$ is decreased by the same amount for \emph{all} global minimizers $\X'$, which then implies that $\X_t$ must converge to the global minimizer which is closest to $\X_0$ in terms of the Bregman divergence. From this observation it follows that $\X_\infty = \lim_{t\rightarrow\infty}\X_t$ minimizes the quantity in \eqref{eq:claim1}
  among all global minimizers of $f$.
  
  The bound \eqref{eq:claim2} on the empirical risk of the last iterate $f(\X_t)$ can be shown using the smoothness of the empirical risk $f$ and the strong convexity of the spectral hypentropy mirror map. A detailed proof of Theorem \ref{theorem:bias} can be found in Appendix \ref{appendix:proofs_theorem_bias}.
  \end{proof}
  We remark that the step size $\eta$ can be chosen independently from the parameter $\beta$ if $\beta$ is chosen from some interval bounded away from infinity, e.g.\ $\beta\in (0,1)$.
  
An analogous result holds for mirror descent equipped with the spectral entropy mirror map \eqref{eq:mirror_map} when optimizing over positive semidefinite matrices.
\begin{theorem}[Positive semidefinite case]
\label{theorem:bias_psd}
Consider the mirror descent algorithm \eqref{eq:mirror_descent} with mirror map \eqref{eq:mirror_map} and initialization $\X_0 = \alpha\mathbf{I}$ for some $\alpha > 0$. Suppose that the step sizes satisfy $\eta_t \equiv \eta \le c$, where $c>0$ is a constant depending on the sensing matrices $\A_i$'s and observations $y_i$'s, and that there exists a positive semidefinite matrix $\X'$ satisfying $f(\X') = 0$. Then, mirror descent converges to a positive semidefinite matrix $\X_\infty = \lim_{t\rightarrow \infty} \X_t$ which, among all global minimizers of $f$, minimizes
\begin{equation}
\label{eq:claim1_psd}
\sum_{i=1}^n\biggl(\log\frac{1}{\alpha} - 1\biggr)\lambda_i+ \lambda_i\log\lambda_i,
\end{equation}
where $\{\lambda_i\}_{i=1}^n$ denote the eigenvalues of $\X_\infty$.
We have, for any $t\ge 0$,
\begin{equation}
\label{eq:claim2_psd}
f(\X_t) \le \frac{D_\Phi(\X_\infty, \X_0)}{\eta t}.
\end{equation}
\end{theorem}
Theorem \ref{theorem:bias_psd} can be proved the same way as Theorem \ref{theorem:bias}, see Appendix \ref{appendix:proofs_theorem_bias_psd}. Neither Theorem requires any assumptions on the sensing matrices $\A_i$'s beyond the existence of a matrix $\X'$ achieving zero training error $f(\X') = 0$, which, for instance, is satisfied when $\{\A_i\}_{i=1}^m$ are linearly independent and $m\le nn'$ in the rectangular case or $m\le n(n+1)/2$ in the positive semidefinite case. 

\textbf{On the implicit bias of mirror descent.}
The implicit bias of mirror descent in linear models has previously been studied using KKT optimality conditions in \cite{GLSS18}. Once we have established convergence of mirror descent towards a global minimizer of $f$, we could have alternatively invoked Theorem 1 of \cite{GLSS18} to characterize the limiting point of mirror descent. Instead, our analysis of the Bregman divergence $D_{\Phi_\beta}(\X',\X_t)$ reveals that each iteration of the mirror descent algorithm \eqref{eq:mirror_descent} decreases $D_{\Phi_\beta}(\X',\X_t)$ by the same amount for all global minmimizers $\X'$ of $f$, provided that
\begin{equation*}
  \langle \nabla f(\X_t), \X_t - \X'\rangle = \langle \nabla f(\X_t), \X_t - \X''\rangle
\end{equation*}
for any two global minimizers $\X'$ and $\X''$ of $f$. From this observation it immediately follows that mirror descent converges to the global minimizer of $f$ that minimizes the quantity $D_{\Phi_\beta}(\X',\X_0)$ among all global minimizers of $f$.
Since this equality is satisfied under the assumptions of Theorem 1 in \cite{GLSS18}, our analysis presents an alternative proof technique for Theorem 1 in \cite{GLSS18} that does not use KKT optimality conditions and hence does not need to assume convergence of mirror descent. 

\textbf{On the mirror map parameter and initialization size.}
In the rectangular case, minimizing the quantity in \eqref{eq:claim1} is equivalent to minimizing the nuclear norm in the limit $\beta \rightarrow 0$, and equivalent to minimizing the Frobenius norm in the limit $\beta \rightarrow \infty$, see Appendix \ref{appendix:proofs_further_claims} for further details. This can be seen as a matrix analogue of the result in \cite{WGL+20}, which shows for linear diagonal networks with shared weights that gradient descent minimizes a quantity interpolating between the $\ell_1$ and (weighted) $\ell_2$ norms. For a network with two layers, this architecture corresponds to the parametrization $\x = \mathbf{u}\odot\mathbf{u} - \mathbf{v}\odot\mathbf{v}$, where $\odot$ denotes the elementwise Hadamard product and with which gradient descent has been shown to be a first-order approximation to mirror descent equipped with the hypentropy mirror map in the vector-case, see e.g.\ \cite{VKR20, WR20b}. In the positive semidefinite case, minimizing the quantity in \eqref{eq:claim1_psd} also corresponds to minimizing the nuclear norm in the limit $\alpha\rightarrow 0$, while in the limit $\alpha\rightarrow\infty$ the coefficient $\log(1/\alpha) - 1$ tends to minus infinity, and minimizing the quantity in \eqref{eq:claim1_psd} is equivalent to maximizing the nuclear norm. 

\section{The estimation problem}
\label{sec:the_estimation_problem}
In this section, we consider the estimation problem of reconstructing a rank-$r$ matrix $\X^\star$ from a set of linear measurements $y_i = \langle \A_i, \X^\star\rangle$, $i = 1,\dots,m$. Theorem \ref{theorem:bias} and Theorem \ref{theorem:bias_psd} imply that mirror descent, equipped with the spectral hypentropy mirror map \eqref{eq:mirror_map_hypentropy} or spectral entropy mirror map \eqref{eq:mirror_map}, approximately minimizes the nuclear norm for a small mirror map parameter $\beta$ or a small initialization size $\alpha$. Hence, mirror descent, a first-order algorithm minimizing the unregularized empirical risk, can recover a low-rank matrix $\X^\star$ under the same set of assumptions which is sufficient for nuclear norm minimization to be successful.

\subsection{Matrix sensing with restricted isometry property}
\label{sec:matrix_sensing_with_rip}
The following restricted isometry property has been shown to be sufficient to guarantee recovery of low-rank matrices using nuclear norm minimization \cite{RFP10}, and has also been used in \cite{LMZ18} to show that gradient descent with full-rank factorized parametrization recovers $\X^\star$.
\begin{definition}[Restricted isometry property \cite{RFP10}]
A set of matrices $\A_1,\dots,\A_m\in\R^{n\times n'}$ satisfies $(r,\delta)$-restricted isometry property (RIP) if for any matrix $\X\in\R^{n\times n'}$ with rank at most $r$, we have
\begin{equation*}
(1 - \delta)\|\X\|_F \le \biggl(\frac{1}{m}\sum_{i=1}^m\langle \A_i,\X \rangle^2\biggr)^{1/2} \le (1 + \delta)\|\X\|_F.
\end{equation*}
\end{definition}
With this definition, we have the following recovery guarantee for mirror descent.
\begin{theorem}
\label{theorem:sensing}
Assume that the set of measurement matrices $\{\A_i\}_{i=1}^m$ satisfies $(5r, \delta)$-restricted isometry property with $\delta \le 0.1$. Then, the mirror descent algorithm described in Theorem \ref{theorem:bias} with parameter $\beta < \frac{\|\X^\star\|_*}{1.05en}$ converges to a matrix $\X_\infty = \lim_{t\rightarrow \infty} \X_t$ that satisfies
\begin{equation}\label{eq:sensing1}
  \|\X_\infty - \X^\star\|_F \le \frac{\Delta_\beta\|\X^\star\|_* + (1 + \Delta_\beta)\frac{n\beta}{\log\frac{\|\X^\star\|_*}{\beta} - 1}}{C_\delta\sqrt{3r}},
\end{equation}
where $C_\delta = \frac{1}{2}(1 - \sqrt{\frac{2}{3}} - \delta(1 + \sqrt{\frac{2}{3}}))$ and $\Delta_\beta = (\frac{\log(\|\X^\star\|_*/\beta) - 1}{\log(1.05n)} - 1)^{-1}$.

If, additionally, $n = n'$ and $\X^\star$ is positive semidefinite, then the mirror descent algorithm described in Theorem \ref{theorem:bias_psd} with $\alpha < \frac{\|\X^\star\|_*}{en}$ converges to a positive semidefinite matrix $\X_\infty$ that satisfies
\begin{equation}\label{eq:sensing2}
  \|\X_\infty - \X^\star\|_F \le \frac{\Delta_\alpha\|\X^\star\|_*}{C_\delta\sqrt{3r}},
  \end{equation}
  where $\Delta_\alpha = (\frac{\log(\|\X^\star\|_*/\alpha) - 1}{\log n} - 1)^{-1}$.
\end{theorem}
Among the existing results on implicit regularization in matrix sensing, the result that is perhaps most closely related to Theorem \ref{theorem:sensing} is Theorem 1.1 in \cite{LMZ18}, which considers the positive semidefinite case of matrix sensing with full-rank parametrization $\X = \U\U^\top$, $\U\in\R^{n\times n}$, and shows that gradient descent recovers $\X^\star$ under the assumption of restricted isometry with constant $\delta \le c /(\kappa^3\sqrt{r}\log^2n)$, where $\kappa$ is the condition number of $\X^\star$ and $c>0$ is an absolute constant.

The analysis that leads to Theorem \ref{theorem:sensing} differs significantly from the proof of Theorem 1.1 in \cite{LMZ18}. The proof of Theorem 1.1 in \cite{LMZ18} involves an analysis of the trajectory of gradient descent, relating it to the population dynamics defined by the population risk 
\begin{equation*}
  \E_{(\A_i)_{kl}\sim \gauss(0,1)}[f(\U_t\U_t^\top)] = \|\U_t\U_t^\top - \X^\star\|_F^2,
\end{equation*}
and identifying adaptive rank-$r$ subspaces to which the iterates $\U_t$ are approximately confined. On the other hand, Theorem \ref{theorem:sensing} follows from Theorem \ref{theorem:bias} and Theorem \ref{theorem:bias_psd}, which rely on a simple analysis of the evolution of the Bregman divergence, and employs arguments similar to the ones used in \cite{RFP10} to prove that RIP is sufficient to guarantee that nuclear norm minimization recovers $\X^\star$. Further, our analysis applies both to the case of rectangular and positive semidefinite matrices, while the results in \cite{LMZ18} only hold for the full-rank parametrization $\X = \U\U^\top$, $\U\in\R^{n\times n}$, for which an extension to general rectangular matrices does not seem to be trivial, see e.g.\ \cite{JKN16, MLC21, SL16, TBSSR16, ZL16}.

\textbf{Sample complexity and condition number.} The restricted isometry property assumption in Theorem 1.1 in \cite{LMZ18} requires a restricted isometry constant $\delta \le c/(\kappa^3\sqrt{r}\log^2n)$ for an absolute constant $c>0$, which leads to a sample complexity of $\O(\kappa^6nr^2\log^5n)$ if the sensing matrices $\A_i$'s belong to a general class of random matrices \cite{RFP10}. It is conjectured in \cite{LMZ18} that the dependence of $\delta$ on the rank $r$ and condition number $\kappa$ is not tight, and that $\delta$ only needs to be smaller than some absolute constant. On the other hand, Theorem \ref{theorem:sensing} only requires $\delta \le 0.1$, which is the same assumption that is sufficient to guarantee that nuclear norm minimization recovers low-rank matrices \cite{RFP10} and leads to a sample complexity of $\O((n + n')r\log (nn'))$ . To the best of our knowledge, our result for mirror descent is the first recovery guarantee for an implicit regularization-based algorithm for (rectangular) matrix sensing that only requires the same assumptions as nuclear norm minimization \cite{RFP10}.

\textbf{Convergence speed and dependence on intialization size.}
The analysis along the trajectory of gradient descent in \cite{LMZ18} allows to establish convergence speed guarantees and a polynomial dependence of the estimation error $\|\U_t\U_t^\top - \X^\star\|_F$ on the initialization size $\alpha$. On the other hand, we have no convergence speed guarantees for mirror descent (beyond the bounds on the empirical risk \eqref{eq:claim2} and \eqref{eq:claim2_psd}), and the final estimation error $\|\X_\infty-\X^\star\|_F$ depends logarithmically on the parameters $\alpha$ and $\beta$, which can lead to an unpractically small value for $\alpha$ and $\beta$ being required to reach some desired accuracy $\varepsilon$. Nonetheless, our result guarantees exact recovery of $\X^\star$ in the limit $\alpha, \beta \rightarrow 0$, which is a setting often considered in the literature, e.g.\ \cite{GWBNS17, LLL21, WGL+20}.

We remark that, beyond the aforementioned result on implicit regularization in matrix sensing \cite{LMZ18}, similar recovery results that include convergence speed guarantees and a polynomial dependence on the initialization size have also been established for implicit regularization-based algorithms in the context of sparse linear regression \cite{VKR19} and sparse phase retrieval \cite{WR20b, WR21}. However, these results all require a sample complexity that scales quadratically in the respective notions of sparsity, and we leave it to future work to investigate whether it is possible to establish convergence guarantees that include a convergence speed analysis and depend polynomially on $\alpha$ and $\beta$ when a sample complexity that scales only linearly in the rank $r$ of $\X^\star$ is assumed.

\subsection{Matrix completion}
In matrix completion, a subset of the entries of the matrix $\X^\star$ is observed, and the corresponding sensing matrices $\A_i$'s do not satisfy the restricted isometry property. Instead, an incoherence condition together with a sufficient number of randomly observed entries have been used to guarantee recovery for nuclear norm minimization \cite{R11} and gradient descent \cite{MWCC18}, for instance.
\begin{definition}[Coherence \cite{CR09}]
  Let $U\subseteq\R^n$ be a linear subspace of dimension $r$ and $\mathbf{P}_U$ the orthogonal projection onto $U$. The coherence of $U$ is defined as ($\{\mathbf{e}_i\}_{i=1}^n$ denotes the canonical basis)
  \begin{equation*}
\mu(U) = \frac{n}{r}\max_{1\le i\le n}\|\mathbf{P}_U\mathbf{e}_i\|_2^2.
  \end{equation*}
\end{definition}
We have the following recovery guarantee for mirror descent under the same assumptions as in \cite{R11}.
\begin{theorem}
\label{theorem:completion}
Let $\X^\star\in\R^{n\times n'}$ be a rank-$r$ matrix with (compact) singular value decomposition $\X^\star = \U\boldsymbol{\Sigma}\V^\top$, where $\U\in\R^{n\times r}$,  $\V\in\R^{n'\times r}$ and $\boldsymbol{\Sigma}\in\R^{r\times r}$. Assume that:

\textbf{A1} The row and column spaces of $\X^\star$ have coherences bounded by $\mu_0>0$.\vspace{-0.5em}

\textbf{A2} The matrix $\U\V^\top$ has maximum entry bounded by $\mu_1\sqrt{r/nn'}$ in absolute value for some \\ \textcolor{white}{.} \hspace{0.7em} $\mu_1>0$.

Suppose that we observe $m$ entries of $\X^\star$ with locations sampled uniformly at random. Then, if $m\ge 32c\max\{\mu_0^2,\mu_1\}r(n + n')\log^2(2n')$ for some $c>1$, the mirror descent algorithm described in Theorem \ref{theorem:bias} with $\beta<\frac{\|\X^\star\|_*}{1.05en}$ converges to a matrix $\X_\infty = \lim_{t\rightarrow\infty}\X_t$ that satisfies
\begin{equation}\label{eq:completion1}
  \|\X_\infty - \X^\star\|_F \le 6\biggl(\Delta_\beta\|\X^\star\|_* + (1+\Delta_\beta)\frac{n\beta}{\log\frac{\|\X^\star\|_*}{\beta}-1}\biggr)\biggl(1 + \biggl(\frac{128cnn'\log^2n'}{9m}\biggr)^{\frac{1}{2}}\biggr),
\end{equation}
with probability at least $1-6\log (n')(n+n')^{2-2c} - (n')^{2-2\sqrt{c}}$, where $\Delta_\beta = (\frac{\log(\|\X^\star\|_*/\beta) - 1}{\log(1.05n)} - 1)^{-1}$.

If, additionally, $n = n'$ and $\X^\star$ is positive semidefinite, then the mirror descent algorithm described in Theorem \ref{theorem:bias_psd} with $\alpha < \frac{\|\X^\star\|_*}{en}$ converges to a positive semidefinite matrix $\X_\infty$ that satisfies
\begin{equation}\label{eq:completion2}
\|\X_\infty - \X^\star\|_F \le 6\Delta_\alpha\|\X^\star\|_*\biggl(1 + \biggl(\frac{128cn^2\log^2n}{9m}\biggr)^{\frac{1}{2}}\biggr),
\end{equation}
where $\Delta_\alpha = (\frac{\log(\|\X^\star\|_*/\alpha) - 1}{\log n} - 1)^{-1}$.
\end{theorem} 

To the best of our knowledge, there are no recovery guarantees in matrix completion for an unregularized empirical risk minimization-based algorithm that does not explicitly enforce the low-rank constraint. For instance, when $\X^\star\in\R^{n\times n}$ is positive semidefinite, Theorem 2 in \cite{MWCC18} establishes a recovery guarantee for gradient descent with low-rank parametrization $\X = \U\U^\top$, where $\U\in\R^{n\times r}$, applied to the unregularized empirical risk with a sample requirement of order $r^3n\log^3n$.

\section{Connection with gradient descent}
\label{sec:connection_with_gradient_descent}
In the vector-case, it has been established that gradient descent with parametrization $\x = \mathbf{u}\odot\mathbf{u} - \mathbf{v}\odot\mathbf{v}$ is a first-order approximation to mirror descent equipped with the hypentropy mirror map \cite{GHS20,VKR20,WR20b}, and a general framework connecting mirror descent and reparametrized gradient descent was studied in \cite{AW20}. A natural question is whether such a connection also extends to the matrix-case. 

In the following, we consider matrix sensing with square symmetric sensing matrices $\A_i\in\R^{n\times n}$, $i = 1,\dots,m$. Consider the following exponentiated gradient algorithm given by
\begin{equation}
  \label{eq:exponentiated_gradient}
  \begin{gathered}
  \X_t = \mathbf{U}_t - \mathbf{V}_t \\
  \U_{t+1} = \frac{\U_te^{-\eta \nabla f(\X_t)} + e^{-\eta \nabla f(\X_t)}\U_t}{2}, \qquad \V_{t+1} = \frac{\V_te^{\eta \nabla f(\X_t)} + e^{\eta \nabla f(\X_t)}\V_t}{2}.
  \end{gathered}
\end{equation}
When considering the full-rank parametrization $\X = \U\U^\top - \V\V^\top$, where $\U,\V\in\R^{n\times n}$, gradient descent on the variables $(\U, \V)$ is a first-order approximation to the exponentiated gradient algorithm defined in \eqref{eq:exponentiated_gradient}, with the step size rescaled by a constant factor and the approximation being exact in the limit $\eta\rightarrow 0$, see Appendix \ref{appendix:proofs_further_claims} for details. The gradient descent algorithm considered in \cite{ACHL19, GWBNS17, LMZ18, LLL21} can be obtained by initializing $\V_0 = \mathbf{0}$.
\begin{proposition}
  \label{prop1}
  Assume that the sensing matrices $\A_i$'s are symmetric and commute. Then:\vspace{-0.8em}
  \begin{enumerate}[leftmargin = 2em]\setlength\itemsep{-0.3em}
    \item Mirror descent equipped with the spectral entropy mirror map \eqref{eq:mirror_map} and any positive definite initialization $\X_0$ which commutes with all $\A_i$'s (e.g.\ $\X_0 = \alpha\mathbf{I}$ for some $\alpha > 0$) is equivalent to the exponentiated gradient algorithm defined in \eqref{eq:exponentiated_gradient} with initialization $\U_0 = \X_0$ and $\V_0 = \mathbf{0}$. 
    \item Mirror descent equipped with the spectral hypentropy mirror map \eqref{eq:mirror_map_hypentropy} with parameter $\beta > 0$ and initialization $\X_0 = \mathbf{0}$ is equivalent to the exponentiated gradient algorithm defined in \eqref{eq:exponentiated_gradient} with initialization $\U_0 = \V_0 = \frac{1}{2}\beta\mathbf{I}$.
  \end{enumerate}
\end{proposition}
As a Corollary, we obtain a generalization of Theorem 1 in \cite{GWBNS17}. Let $\widetilde{f}(\U) = f(\U\U^\top)$.
\begin{corollary}\label{coroll}
  Assume that the sensing matrices $\A_i$'s are symmetric and commute, and that there exists a $\X'\in\S^n_+$ satisfying $f(\X') = 0$. Then, the gradient flow defined by $\frac{d\U_t}{dt} = -\nabla \widetilde{f}(\U_t)$ and any initialization satisfying $\U_0\U_0^\top = \alpha\mathbf{I}$ converges to a matrix $\U_\infty$ minimizing
  \begin{equation*}
    \sum_{i=1}^n\biggl(\log\frac{1}{\alpha} - 1\biggr)\lambda_i+ \lambda_i\log\lambda_i
  \end{equation*}
  among all global minimizers of $\widetilde{f}$, where $\{\lambda_i\}_{i=1}^n$ denote the eigenvalues of the matrix $\U_\infty\U_\infty^\top$.
\end{corollary}
This result generalizes Theorem 1 in \cite{GWBNS17} in two ways: first, we obtain a precise characterization of the quantity that is minimized for any initialization size $\alpha > 0$, which indeed coincides with the nuclear norm in the limit $\alpha\rightarrow 0$. Second, convergence to a global minimizer is assumed in Theorem 1 in \cite{GWBNS17}, which is non-trivial a-priori, since the optimization problem in $\U_t$ is non-convex. On the other hand, we show convergence of mirror descent in Theorem \ref{theorem:bias_psd}, which then carries over to gradient flow on the non-convex objective $\widetilde{f}$ via Proposition \ref{prop1} (when the $\A_i$'s are symmetric and commute).

\section{Numerical simulations}
\label{sec:numerical_simulations}
In this section, we present numerical simulations examining the dependence of the final estimates of mirror descent equipped with the spectral entropy \eqref{eq:mirror_map} and of gradient descent with full-rank parametrization $\X = \U\U^\top$, $\U\in\R^{n\times n}$, on the initialization size for random Gaussian sensing matrices $\A_1,\dots,\A_m$. We evaluate the nuclear norm $\|\X\|_*$, the reconstruction error $\|\X-\X^\star\|_F$, and the effective rank \cite{RV07} $\operatorname{effrank}(\X) = \exp(-\sum_{i=1}^np_i\log p_i)$, where $p_i = \sigma_i / \|\X\|_*$, $i=1,\dots, n$, denote the normalized singular values of $\X$. Numerical simulations for matrix completion are provided in Appendix \ref{appendix:additional_experiments} and yield similar results as for random Gaussian sensing matrices.

Our experimental setup is as follows. We generate a rank-$r$ positive semidefinite matrix by sampling a random matrix $\U^\star\in\R^{n\times r}$ with i.i.d.\ $\gauss(0,1)$ entries, setting $\X^\star = \U^\star(\U^\star)^\top$ and normalizing $\|\X^\star\|_* = 1$. We generate $m$ symmetric sensing matrices $\A_i = \frac{1}{2}(\mathbf{B}_i + \mathbf{B}_i^\top)$, where the entries of $\mathbf{B}_i$ are i.i.d.\ $\gauss(0,1)$. We run mirror descent and gradient descent with initialization $\X_0 = \alpha\mathbf{I}$ and constant step sizes $\mu = 1$ and $\mu = 0.25$, respectively, for $T = 5000$ iterations, and vary the initialization size $\alpha$ from $10^{-1}$ to $10^{-10}$. For reference, we also include the ground truth $\X^\star$ and the estimate $\operatorname{argmin}\{\|\X\|_*: \X\succeq \mathbf{0}, f(\X) = 0\}$ obtained from minimizing the nuclear norm using the \texttt{cvxopt} package. The experiments for Figure \ref{figure1} were implemented in Python 3.9 and took around $10$ minutes on a machine with 1.1-GHz Intel Core i5 CPU and 8 GB of RAM.
\begin{figure}[!htb]
\centering
\includegraphics[width=0.925\textwidth]{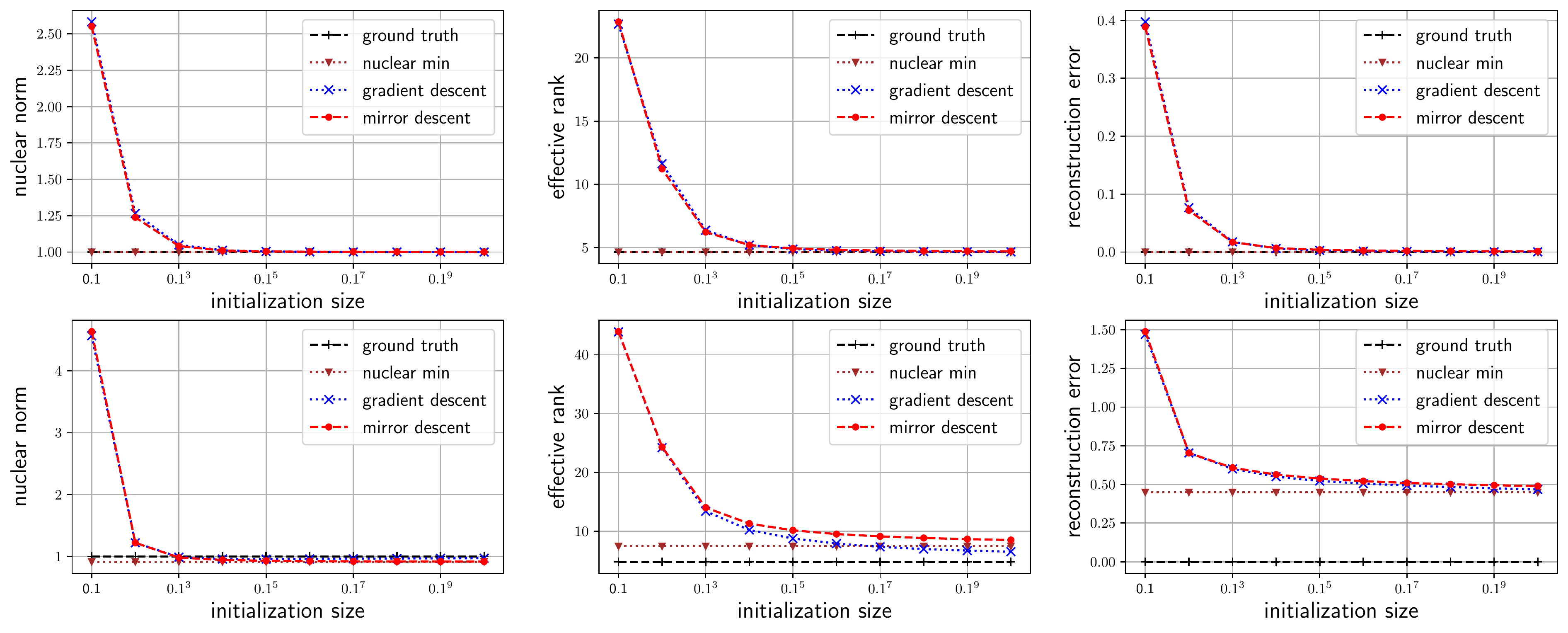}
\caption{Nuclear norm, effective rank \cite{RV07} and reconstruction error against initialization size $\alpha$ for $n=50$ and $r=5$. Top row: $m=3nr$ measurements. Bottom row: $m=nr$ measurements.}\vspace{-0.5em}
\label{figure1}
\end{figure}

With $m=3nr$ measurements (Figure \ref{figure1}, top row), nuclear norm minimization coincides with $\X^\star$. In this case, our simulations show that the estimate of gradient descent closely tracks the estimate of mirror descent for all initialization sizes in terms of nuclear norm, effective rank \cite{RV07} and reconstruction error, even though the sensing matrices do not commute. When we have $m=nr$ measurements (Figure \ref{figure1}, bottom row), nuclear norm minimization does not recover the planted matrix $\X^\star$, and our simulations show that gradient descent puts more emphasis on lowering the effective rank at the expense of a higher nuclear norm for initialization sizes smaller than $10^{-3}$. Since Theorem \ref{theorem:bias_psd} guarantees that mirror descent minimizes the nuclear norm in the limit $\alpha\rightarrow 0$, regardless of the number of measurements, this is in line with the observations in \cite{ACHL19, LLL21}, which suggest that a notion of rank-minimization is involved in the implicit bias of gradient descent. 

\section{Conclusion}
\label{sec:conclusion}
In this paper, we analyzed discrete-time mirror descent for matrix sensing, equipped with the spectral hypentropy mirror map in the case of general rectangular matrices and equipped with the spectral entropy mirror map in the particular case of positive semidefinite matrices. We showed that mirror descent minimizes a quantity that interpolates between the nuclear norm and Frobenius norm in the rectangular case, and is a linear combination of the nuclear norm and the negative von Neumann entropy in the positive semidefinite case. We used this result to show that mirror descent, a first-order algorithm minimizing the unregularized empirical risk that does not explicitly enforce low-rankness of its iterates, can recover a low-rank matrix $\X^\star$ under the same set of assumptions which is sufficient for nuclear norm minimization to recover $\X^\star$ \cite{R11, RFP10}. 

A downside of mirror descent compared to gradient descent with full-rank factorized parametrization, which is an alternative implicit regularization-based algorithm for matrix sensing, is its computational cost: the spectral hypentropy requires a singular value decomposition in each iteration, while the spectral entropy requires computing a matrix exponential in each iteration, see Appendix \ref{appendix:mirror_maps_gradients}. For general sensing matrices, the computational cost of mirror descent is of the same order as that of gradient descent, since a singular value decomposition takes $\O(n^2n')$ operations and matrix exponentials can be computed in $\O(n^3)$ operations, while evaluating the gradient $\nabla f(\X) = \frac{1}{m}\sum_{i=1}^m(\langle\A_i, \X\rangle - y_i)\A_i$ requires $\O(mnn')$ operations, and $m>n$ is typically required to recover $\X^\star$. However, the inner product $\langle \A_i, \X\rangle$ can be computed in $\O(1)$ operations in matrix completion, in which case the computational cost of gradient descent with full-rank factorized parametrization is dominated by the multiplication of two $n\times n$ matrices. When $r\ll \min\{n,n'\}$, neither implicit regularization-based approach achieves the computational efficiency of gradient descent with low-rank factorized parametrization $\X = \U\V^\top$, where $\U\in\R^{n\times r}$, $\V\in\R^{n'\times r}$, see e.g.\ \cite{CLC19, MWCC18}.

Previous results on implicit regularization-based algorithms in matrix sensing \cite{LMZ18}, sparse linear regression \cite{VKR19} and sparse phase retrieval \cite{WR20b, WR21} establish convergence speed guarantees and a polynomial dependence on the initialization size $\alpha$, while a sample complexity that scales quadratically in the respective notions of sparsity is required. On the other hand, our results for mirror descent in matrix sensing require a sample complexity that scales linearly in the rank $r$ of $\X^\star$, but do not establish any convergence speed guarantees, and the bound on the estimation error depends logarithmically on $\alpha$ and $\beta$. We leave bridging this gap, i.e.\ establishing convergence speed guarantees with a polynomial dependence on $\alpha$ and $\beta$ while only assuming a linear sample complexity, to future work.

\section*{Acknowledgments and Disclosure of Funding}
Fan Wu is supported by the EPSRC and MRC through the OxWaSP CDT programme (EP/L016710/1). Patrick Rebeschini was supported in part by the Alan Turing Institute under the EPSRC grant EP/N510129/1.

\bibliography{references}
\bibliographystyle{plainnat}

\clearpage
\appendix

{\LARGE\textbf{Appendix}}
\vspace{3mm}

The Appendix is organized as follows. In Appendix \ref{appendix:mirror_maps_gradients}, we provide the gradients for the spectral entropy \eqref{eq:mirror_map} and spectral hypentropy \eqref{eq:mirror_map_hypentropy} mirror maps, and discuss the per-iteration computational cost of the corresponding mirror descent algorithms. In Appendix \ref{appendix:proofs}, we provide proofs for the claims made in the main paper. In Appendix \ref{appendix:additional_experiments}, we present additional experiments addressing the problem of matrix completion. 

\section{Mirror maps and gradients}
\label{appendix:mirror_maps_gradients}
We first consider rectangular matrix sensing with the spectral hypentropy mirror map 
\begin{equation*}
  \Phi_\beta(\X) = \sum_{i=1}^{n}\sigma_i \operatorname{arcsinh}\biggl(\frac{\sigma_i}{\beta}\biggr) - \sqrt{\sigma_i^2 + \beta^2},
\end{equation*}
where $\{\sigma_i\}_{i=1}^n$ denote the singular values of $\X$.
Since $\Phi_\beta:\R^{n\times n'}\rightarrow \R$ is a function operating on the singular values of a matrix, we can use Theorem 3.1 in \cite{L95} to compute its gradient. Let $\X = \U\operatorname{diag}(\sigma_1,\dots,\sigma_n)\V^\top$ be the singular value decomposition of the matrix $\X$, where $\operatorname{diag}(\sigma_1,\dots,\sigma_n)$ denotes the diagonal matrix with diagonal elements $\sigma_1,\dots,\sigma_n$. Then, we have
\begin{equation*}
\nabla\Phi_\beta(\X) = \U\operatorname{diag}\Bigl(\operatorname{arcsinh}\Bigl(\frac{\sigma_1}{\beta}\Bigr),\dots,\operatorname{arcsinh}\Bigl(\frac{\sigma_n}{\beta}\Bigr)\Bigr)\V^\top
\end{equation*}
by Theorem 3.1 in \cite{L95}, see also \cite{GHS20}. This means that each step of mirror descent requires a singular value decomposition to compute 
\begin{equation*}
  \X_{t+1} = \nabla\Phi_\beta^{-1}\bigl(\nabla\Phi_\beta(\X_t) - \eta\nabla f(\X_t)\bigr),
\end{equation*}
which takes $\O(n^2n')$ operations. 

The singular value decomposition can be avoided if $n=n'$ and the sensing matrices $\A_i$'s are symmetric, which we can assume without loss of generality if $\X^\star\in\R^{n\times n}$ is symmetric, since then $y_i = \langle\A_i,\X^\star\rangle = \langle \frac{1}{2}(\A_i + \A_i^\top), \X^\star\rangle$ for all $i = 1,\dots,m$. In that case, the mirror descent iterates $\X_t$ stay symmetric for all $t\ge 0$, provided the initialization $\X_0$ is symmetric. Using the identity $\operatorname{arcsinh}(x) = \log(x + \sqrt{x^2 + 1})$, we can write
\begin{equation*}
  \Phi_\beta(\X_t) = \tr\Biggl(\X_t \log \Biggl(\frac{\X_t}{\beta} + \sqrt{\frac{\X_t^2}{\beta^2} + \mathbf{I}}\Biggr) - \sqrt{\X_t^2 + \beta^2\mathbf{I}}\Biggr),
    \end{equation*}
    since all matrices in above expression are symmetric and simultaneously diagonalizable. In this case, the gradient of the spectral hypentropy can be written as
    \begin{equation*}
  \nabla \Phi_\beta(\X_t) = \log \Biggl(\frac{\X_t}{\beta} + \sqrt{\frac{\X_t^2}{\beta^2} + \mathbf{I}}\Biggr),
    \end{equation*}
    and its inverse is given by
  \begin{equation*}
  \nabla\Phi_\beta^{-1}(\X_t) = \beta\frac{e^{\X_t} - e^{-\X_t}}{2}.
  \end{equation*}
Hence, for the mirror descent algorithm \eqref{eq:mirror_descent} we need to compute two matrix exponentials in each iteration. While computing matrix exponentials require $\O(n^3)$ operations, which is of the same order as a singular value decomposition, matrix exponentials are typically cheaper to compute in practice.

In the positive semidefinite case, the spectral entropy mirror map is given by 
\begin{equation*}
\Phi(\X) = \tr(\X\log \X - \X),
\end{equation*}
which has gradient given by
\begin{equation*}
\nabla\Phi(\X) = \log \X,
\end{equation*}
with inverse
\begin{equation*}
\nabla\Phi^{-1}(\X) = \exp (\X).
\end{equation*}
Hence, mirror descent equipped with the spectral entropy mirror map requires computing a matrix exponential in each iteration, which requires $\O(n^3)$ operations. 

\section{Proofs}
\label{appendix:proofs}
In this section, we provide proofs for the claims made in the main paper.
\subsection{Proof of Theorem \ref{theorem:bias}}
\label{appendix:proofs_theorem_bias}
\begin{proof}
  We begin by showing convergence of mirror descent to a global minimizer of the empirical risk $f$. The characterization of the limiting point follows immediately from the proof of convergence. Then, we show the bound \eqref{eq:claim2} by showing that the empirical risk $f(\X_t)$ is monotonously decreasing.

  \textbf{Part 1: Convergence of mirror descent.}\\
  The following identity characterizes the evolution of the Bregman divergence and follows from its definition \eqref{eq:bregman_divergence} and the mirror descent update \eqref{eq:mirror_descent}:
  \begin{equation}
  \label{eq:bregman_difference}
  D_{\Phi_\beta}(\X', \X_{t+1}) - D_{\Phi_\beta}(\X', \X_t) = - \eta\langle \nabla f(\X_t), \X_t - \X'\rangle + D_{\Phi_\beta}(\X_t, \X_{t+1}),
  \end{equation}
  where $\X'$ is any reference point. Letting $\X'$ be any global minimizer of $f$, the first term in \eqref{eq:bregman_difference} can be written as
  \begin{equation*}
  \langle \nabla f(\X_t), \X_t - \X'\rangle = \frac{1}{m}\sum_{i=1}^m\bigl(\langle\A_i, \X_t\rangle - y_i\bigr)\langle\A_i, \X_t - \X'\rangle = 2f(\X_t),
  \end{equation*}
  where we used the assumption that there exists a matrix achieving zero training error, i.e.\ $\langle\A_i, \X'\rangle = y_i$ for all $i=1,\dots,m$. The spectral hypentropy mirror map is $(2(\tau + \beta n))^{-1}$-strongly convex with respect to the nuclear norm $\|\cdot\|_*$ on the nuclear norm ball $\mathcal{B}(\tau) = \{\X\in\R^{n\times n'}: \|\X\|_*\le \tau\}$, see Theorem 14 in \cite{GHS20}. Writing $\tau_t = \max\{\|\X_t\|_*, \|\X_{t+1}\|_*\}$, we can bound the second term in \eqref{eq:bregman_difference} by
  \begin{align}
  \label{eq:bregman_step}
  D_{\Phi_\beta}(\X_t, \X_{t+1}) &= \Phi_\beta(\X_t) - \Phi_\beta(\X_{t+1}) - \langle \nabla\Phi_\beta(\X_{t+1}), \X_t - \X_{t+1}\rangle \nonumber\\
  &\le \langle \nabla\Phi_\beta(\X_t) - \nabla\Phi_\beta(\X_{t+1}), \X_t - \X_{t+1}\rangle - \frac{1}{4(\tau_t + \beta n)} \|\X_t - \X_{t+1}\|_*^2 \nonumber\\
  &= \langle \eta\nabla f(\X_t), \X_t - \X_{t+1} \rangle - \frac{1}{4(\tau_t + \beta n)}\|\X_t - \X_{t+1}\|_*^2 \nonumber\\
  &\le \eta \|\nabla f(\X_t)\|_2\|\X_t - \X_{t+1}\|_* - \frac{1}{4(\tau_t + \beta n)}\|\X_t-\X_{t+1}\|_*^2 \nonumber\\
  &\le \eta^2(\tau_t + \beta n)\|\nabla f(\X_t)\|_2^2,
  \end{align}
  where we used strong convexity of $\Phi_\beta$ in the second line, the mirror descent update \eqref{eq:mirror_descent} in the third line, the fact that the spectral norm $\|\cdot\|_2$ is the dual norm to the nuclear norm $\|\cdot\|_*$ in the fourth line, and we optimized a quadratic function in $\|\X_t-\X_{t+1}\|_*$ to obtain the last inequality.
  
  The spectral norm of the gradient $\nabla f$ can be bounded in terms of the empirical risk $f$: we have
  \begin{align*}
  \|\nabla f(\X_t)\|_2^2 &= \biggl\|\frac{1}{m}\sum_{i=1}^m\bigl(\langle \A_i, \X_t \rangle - y_i \bigr)\A_i \biggr\|_2^2 \\
  &\le \Biggl(\frac{1}{m}\sum_{i=1}^m\bigl|\langle \A_i, \X_t\rangle - y_i\bigr|\|\A_i\|_2\Biggr)^2 \\
  &\le \frac{1}{m}\sum_{i=1}^m \|\A_i\|_2^2 \cdot 2f(\X_t),
  \end{align*}
  where we used the triangle inequality in the second and the Cauchy-Schwarz inequality in the last line. Using the non-negativity of the Bregman divergence, we can rearrange the penultimate inequality in \eqref{eq:bregman_step} to obtain 
  \begin{equation}
  \label{eq:difference_step}
  \|\X_t-\X_{t+1}\|_* \le 4(\tau_t + \beta n)\eta\|\nabla f(\X_t)\|_2 \le \frac{1}{2}(\tau_t + \beta n),
  \end{equation}
  provided the step size $\eta$ satisfies
  \begin{equation}
    \label{eq:stepsize1}
  \eta \le \frac{1}{8\sqrt{2}}\Biggl(\frac{1}{m}\sum_{i=1}^m\|\A_i\|_2^2\cdot f(\X_t)\Biggr)^{-1/2}. 
  \end{equation}
  We will show below that the upper bound in \eqref{eq:stepsize1} is uniformly bounded from below by a constant $c > 0$, i.e.\ we can indeed choose a constant step size $\eta_t\equiv\eta \le c$. If $\|\X_{t+1}\|_* > \|\X_t\|_*$, then the reverse triangle inequality yields
  \begin{equation*}
  \|\X_{t+1}\|_* - \|\X_t\|_* \le \|\X_t - \X_{t+1}\|_* \le \frac{1}{2}\Bigl(\|\X_{t+1}\|_* + \beta n\Bigr),
  \end{equation*}
  which can be rearranged to $\|\X_{t+1}\|_*\le 2\|\X_t\|_* + \beta n$, so that also $\tau_t \le 2\|\X_t\|_* + \beta n$.
  Hence, the second term in \eqref{eq:bregman_difference} can be bounded by
  \begin{equation*}
  D_{\Phi_\beta}(\X_t, \X_{t+1}) \le \eta^2(\tau_t + \beta n)\|\nabla f(\X_t)\|_2^2 \le \eta f(\X_t),
  \end{equation*}
  provided that the step size $\eta$ also satisfies
  \begin{equation}
    \label{eq:stepsize2} 
    \eta \le \frac{1}{4}\Biggl(\frac{1}{m}\sum_{i=1}^m\|\A_i\|_2^2 \cdot \Bigl(\|\X_t\|_* + \beta n\Bigr)\Biggr)^{-1}.
  \end{equation}
  With this, the identity in \eqref{eq:bregman_difference} becomes
  \begin{equation}
  \label{eq:bregman_difference2}
  D_{\Phi_\beta}(\X', \X_{t+1}) - D_{\Phi_\beta}(\X',\X_t) = -2\eta f(\X_t) + D_{\Phi_\beta}(\X_t,\X_{t+1}) \le -\eta f(\X_t)
  \end{equation}
  for any global minimizer $\X'$ of $f$. Since the Bregman divergence $D_{\Phi_\beta}(\X',\X_t)$ is bounded from below by zero, this means that the empirical risk $f(\X_t)$ must converge to zero, which in turn implies that $\X_t$ converges to a global minimizer of $f$. 
  
  To see \emph{which} global minimizer mirror descent converges to, observe that the difference in \eqref{eq:bregman_difference_sketch} does not depend on the reference point $\X'$, as long as $\X'$ is a global minimizer of $f$. This means that the Bregman divergence $D_{\Phi_\beta}(\X', \X_t)$ is decreased by the same amount for \emph{all} global minimizers $\X'$, which then implies that $\X_t$ must converge to the global minimizer which is closest to $\X_0$ in terms of the Bregman divergence. Hence, writing $\{\sigma_i\}_{i=1}^n$ for the singular values of $\X_\infty = \lim_{t\rightarrow\infty}\X_t$ and using the identity $\operatorname{arcsinh}(x) = \log(x + \sqrt{x^2 + 1})$, the quantity
  \begin{align*}
  D_{\Phi_\beta}(\X_\infty, \X_0) &= \sum_{i=1}^{n}\sigma_i\operatorname{arcsinh}\biggl(\frac{\sigma_i}{\beta}\biggr) - \sqrt{\sigma_i^2 + \beta^2} - n\beta\\
  &= \sum_{i=1}^{n} \sigma_i\log\frac{1}{\beta} + \sigma_i\log\Bigl(\sigma_i + \sqrt{\sigma_i^2 + \beta^2}\Bigr) - \sqrt{\sigma_i^2 + \beta^2} - n\beta
  \end{align*}
  is minimized among all global minimizers of the empirical risk $f$, which is the quantity in \eqref{eq:claim1} modulo the constant $n\beta$.

  Finally, since we show convergence of $\X_t$, this means that the nuclear norm $\|\X_t\|_*$ and the empirical risk $f(\X_t)$ stay bounded for all $t\ge 0$. This implies that, in order to satisfy inequalities \eqref{eq:stepsize1} and \eqref{eq:stepsize2}, we can indeed choose a constant step size $\eta_t \equiv \eta\le c$, where the constant $c>0$ depends on the spectral norm of the sensing matrices $\A_i$'s and the observations $y_i$'s.

  \textbf{Part 2: Proving the bound \eqref{eq:claim2}.}\\  
  In order to show the bound \eqref{eq:claim2}, we first show that $f(\X_t)$ decreases monotonously. To this end, we verifiy that $f$ is $\frac{1}{m}\sum_{i=1}^m\|\A_i\|_2^2$-smooth with respect to the nuclear norm. Indeed, $\nabla f$ is Lipschitz continuous with Lipschitz constant $\frac{1}{m}\sum_{i=1}^m\|\A_i\|_2^2$, 
  \begin{align*}
  \|\nabla f(\X) - \nabla f(\Y)\|_2 &= \biggl\|\frac{1}{m}\sum_{i=1}^m\langle \A_i, \X - \Y \rangle\A_i \biggr\|_2 \\
  &\le \biggl\|\frac{1}{m}\sum_{i=1}^m\|\A_i\|_2\|\X-\Y\|_* \A_i \biggr\|_2 \\
  &\le \frac{1}{m}\sum_{i=1}^m\|\A_i\|_2^2 \cdot \|\X-\Y\|_*,
  \end{align*}
  where we used the duality of the nuclear and spectral norms in the second line. Hence, we can bound
  \begin{equation*}
  f(\X_{t+1}) \le f(\X_t) + \langle \nabla f(\X_t), \X_{t+1} - \X_t\rangle + \frac{1}{2m}\sum_{i=1}^m\|\A_i\|_2^2 \cdot \|\X_{t+1} - \X_t\|_*^2.
  \end{equation*}
  If we can bound 
  \begin{equation*}
  \langle \nabla f(\X_t), \X_{t+1} - \X_t\rangle \le -\frac{1}{2m}\sum_{i=1}^m\|\A_i\|_2^2 \cdot \|\X_{t+1} - \X_t\|_*^2,
  \end{equation*}
  then this would show that $f(\X_t)$ is monotonously decreasing. 
  Recall the proximal formulation of mirror descent (see e.g.\ \cite{BT03}),
  \begin{equation*}
  \X_{t+1} = \operatorname{arg}\min_{\X\in \R^{n\times n'}}\biggl\{\langle\nabla f(\X_t), \X - \X_t\rangle + \frac{1}{\eta}D_{\Phi_\beta}(\X, \X_t) \biggr\}.
  \end{equation*}
  Since the quantity being minimized is zero for $\X = \X_t$, we obtain the upper bound
  \begin{align*}
  \langle \nabla f(\X_t), \X_{t+1} - \X_t\rangle &\le -\frac{1}{\eta}D_{\Phi_\beta}(\X_{t+1}, \X_t) \\
  &= \frac{1}{\eta}\Bigl((\Phi_\beta(\X_t) - \Phi_\beta(\X_{t+1})) + \langle\nabla\Phi_\beta(\X_t), \X_{t+1} - \X_t\rangle\Bigr) \\
  &\le - \frac{1}{4\eta (\tau_t + \beta n)}\|\X_{t+1} - \X_t\|_*^2 \\
  &\le -\frac{1}{2m}\sum_{i=1}^m\|\A_i\|_2^2\cdot \|\X_{t+1} - \X_t\|_*^2,
  \end{align*}
  where we used strong convexity of $\Phi_\beta$ for the second inequality, and the last inequality holds if the step size $\eta$ satisfies inequality \eqref{eq:stepsize2}. This completes the proof that $f(\X_{t+1})\le f(\X_t)$. 
  
  To show the bound \eqref{eq:claim2}, assume that it were violated for some $t>0$. Since $f(\X_t)$ is non-increasing, this means that
  \begin{equation*}
  f(\X_s) \ge f(\X_t) > \frac{D_{\Phi_\beta}(\X_\infty, \X_0)}{\eta t}
  \end{equation*}
  for all $s\le t$. The bound in \eqref{eq:bregman_difference2} controls by how much the Bregman divergence must decrease in each iteration. Summing over the expression in \eqref{eq:bregman_difference2}, we obtain
  \begin{align*}
  D_{\Phi_\beta}(\X_\infty, \X_t) &= D_{\Phi_\beta}(\X_\infty, \X_0) + \sum_{s=0}^{t-1}D_{\Phi_\beta}(\X_\infty, \X_{s+1}) - D_{\Phi_\beta}(\X_\infty, \X_s) \\
  &< D_{\Phi_\beta}(\X_\infty, \X_0) - \sum_{s=0}^{t-1}\eta \frac{D_{\Phi_\beta}(\X_\infty, \X_0)}{\eta t} \\
  &=0,
  \end{align*}
  which contradicts the non-negativity of the Bregman divergence and therefore shows that the bound \eqref{eq:claim2} must be satisfied for all $t>0$.
\end{proof}

\subsection{Proof of Theorem \ref{theorem:bias_psd}}
\label{appendix:proofs_theorem_bias_psd}
The proof of Theorem \ref{theorem:bias_psd} follows the same steps as the proof of Theorem \ref{theorem:bias} and uses the fact that the spectral entropy \eqref{eq:mirror_map} is $(2\tau)^{-1}$ strongly convex with respect to the nuclear norm on the nuclear norm ball $\mathcal{B}_+(\tau) = \{\X\in\S^n_+: \|\X\|_*\le \tau\}$, for which we include a proof for completeness' sake.

\begin{lemma}[Strong convexity of the spectral entropy]
  \label{lemma:strong_convexity}
 The spectral entropy \eqref{eq:mirror_map} is $(2\tau)^{-1}$-strongly convex with respect to the nuclear norm $\|\cdot\|_*$ on the nuclear norm ball $\mathcal{B}_+(\tau)$.
\end{lemma}

\begin{proof}
The proof of Lemma \ref{lemma:strong_convexity} closely follows the proof of strong convexity of the spectral hypentropy mirror map provided in \cite{GHS20}. We first introduce some notation. We denote by $\lambda (\X)$ the vector of eigenvalues of a symmetric matrix $\X\in\S^n$. For a function $f:\R\rightarrow \R$, we denote by $f(\X)$ the standard lifting of scalar functions to symmetric matrices, see e.g.\ \cite{GHS20},
\begin{equation*}
\X = \U\operatorname{diag}[\lambda(\X)]\U^\top \qquad \Rightarrow \qquad f(\X) = \U\operatorname{diag}[f(\lambda(\X))]\U^\top,
\end{equation*}
where $f$ is applied to the vector $\lambda(\X)$ componentwise.

In order to show that the spectral entropy $\Phi$ is $(2\tau)^{-1}$-strongly convex with respect to the nuclear norm $\|\cdot\|_*$ on $\mathcal{B}_+(\tau)$, we use the duality of strong convexity and smoothness and show instead that the Fenchel conjugate $\Phi^*$ is $2\tau$-smooth with respect to the spectral norm $\|\cdot\|_2$ on the set $\nabla\Phi(\mathcal{B}_+(\tau))$.

The following Theorem from \cite{KST12} relates the conjugate of rotationally invariant matrix functions, i.e.\ functions that can be written as $\Psi(\X) = (\psi\circ\lambda)(\X)$, where $\psi:\R^n\rightarrow \R$, to the conjugate of the vector function $\psi$.
\begin{theorem}[Theorem 28 \cite{KST12}]
Let $g:\R^n\rightarrow \R$ be a symmetric function, i.e.\ invariant under permutations of its argument. Then, 
\begin{equation*}
(g\circ\lambda)^* = g^*\circ\lambda.
\end{equation*}
\end{theorem}
With this, we can compute the Fenchel conjugate of $\Phi$. We have
\begin{equation*}
\Phi(\X) = \tr(\X\log\X - \X) = \sum_{i=1}^n\lambda_i(\X)\log\lambda_i(\X) - \lambda_i(\X) =: \sum_{i=1}^n\phi(\lambda_i(\X)),
\end{equation*}
so that
\begin{equation*}
\Phi^*(\X) = \sum_{i=1}^n\phi^*(\lambda_i(\X)) = \sum_{i=1}^ne^{\lambda_i(\X)},
\end{equation*}
since the conjugate of the scalar function $\phi(x) = x\log x - x$ is given by $\phi^*(x)= e^x$.

The following Lemma from \cite{JM08} allows us to reduce the smoothness of matrix functions to the smoothness of functions taking vectors as argument. 
\begin{lemma}[Proposition 3.1 \cite{JM08}]
\label{lemma:a1}
Let $f:\R_+\rightarrow \R$ be a twice continuously differentiable function and $c>0$ a constant such that, for all $b>a> 0$,
\begin{equation*}
\frac{f'(b) - f'(a)}{b-a} \le c\frac{f''(a) + f''(b)}{2}.
\end{equation*} 
Then, the function $F:\S^n\rightarrow \R$ defined by $F(\X) = \tr(f(\X))$ is twice continuously differentiable and satisfies, for every $\mathbf{H}\in\S^n$,
\begin{equation*}
D^2F(\X)[\mathbf{H}, \mathbf{H}] \le c\tr(\mathbf{H}f''(\X)\mathbf{H}).
\end{equation*}
\end{lemma}
We can now analyze the smoothness of $\Phi^*$. By the mean-value theorem, we have for some $c\in[a, b]$,
\begin{equation*}
\frac{(\phi^*)'(b) - (\phi^*)'(a)}{b-a} = (\phi^*)''(c) \le (\phi^*)''(a) + (\phi^*)''(b).
\end{equation*}
Then, by Lemma \ref{lemma:a1}, we can bound, for any $\X = \nabla\Phi(\Y)$ with $\Y\in\mathcal{B}_+(\tau)$,
\begin{align*}
\sup_{\mathbf{H}\in\S^n: \|\mathbf{H}\|_2\le 1}D^2\Phi^*(\X)[\mathbf{H},\mathbf{H}] &\le \sup_{\mathbf{H}\in\S^n: \|\mathbf{H}\|_2\le 1} 2\tr(\mathbf{H}(\phi^*)''(\X)\mathbf{H}) \\
&= \sup_{\mathbf{H}\in\S^n: \|\mathbf{H}\|_2\le 1} 2\tr(\mathbf{H}^2(\phi^*)''(\X)) \\
&\le \sup_{\mathbf{H}\in\S^n: \|\mathbf{H}\|_2\le 1} 2\langle \sigma^2(\mathbf{H}), \sigma((\phi^*)''(\X))\rangle,
\end{align*}
where we write $\sigma(\X)$ for the vector of singular values of a matrix $\X$. The equality follows from commutativity of the trace, and the last inequality follows from von Neumann's trace inequality $\tr(\A^\top\mathbf{B})\le \langle \sigma(\A), \sigma(\mathbf{B})\rangle$. By definition, we have $\sigma_i^2(\mathbf{H})\le 1$ and $\sigma_i((\phi^*)''(\X)) = \sigma_i(\Y)$ for all $i=1,\dots,n$, so that we can bound
\begin{equation*}
\sup_{\mathbf{H}\in\S^n: \|\mathbf{H}\|_2\le 1}D^2\Phi^*(\X)[\mathbf{H},\mathbf{H}] \le 2\sum_{i=1}^n1 \cdot \sigma_i(\Y) = 2\|\Y\|_* \le 2\tau,
\end{equation*}
which completes the proof that $\Phi^*$ is $2\tau$-smooth with respect to the spectral norm on $\nabla\Phi(\mathcal{B}_+(\tau))$.
\end{proof}

Since the rest of the proof of Theorem \ref{theorem:bias_psd} follows the exact same steps as the proof of Theorem \ref{theorem:bias}, it is omitted to avoid repetition.

\subsection{Proof of Theorem \ref{theorem:sensing}}
  \begin{proof}[Proof of Theorem \ref{theorem:sensing}]
    We begin by considering the rectangular case and prove the bound \eqref{eq:sensing1} in Theorem \ref{theorem:sensing} for the spectral hypentropy mirror map \eqref{eq:mirror_map_hypentropy}. 
    The proof of Theorem \ref{theorem:sensing} is an adaption of and builds upon the proofs of Theorem 3.3 in \cite{RFP10} and Theorem 4 in \cite{FCRP08}. First, we need to bound the nuclear norm of the matrix $\X_\infty$. Then, we follow \cite{FCRP08,RFP10} and use the RIP-assumption to bound the deviation $\|\X_\infty - \X^\star\|_F$. 
    
    \textbf{Step 1: Bound the nuclear norm $\|\X_\infty\|_*$.}\\
    If $\|\X_\infty\|_*\le \|\X^\star\|_*$, then we have a suitable upper bound for the nuclear norm $\|\X_\infty\|_*$. Hence, assume that $\|\X_\infty\|_* > \|\X^\star\|_*$. By Theorem \ref{theorem:bias}, $\X_\infty$ minimizes the quantity in \eqref{eq:claim1} among all global minimizers of the empirical risk $f$ which, in particular, include $\X^\star$. Writing $\sigma_i$ and $\mu_i$ for the singular values of $\X_\infty$ and $\X^\star$, respectively, we can bound
    \begin{align*}
      &\sum_{i=1}^{n}\sigma_i\log\frac{\|\X^\star\|_*}{\beta} + \sigma_i\log\frac{\sigma_i + \sqrt{\sigma_i^2 + \beta^2}}{\|\X^\star\|_*} - \sqrt{\sigma_i^2 + \beta^2} \\
      \le &\sum_{i=1}^{n}\mu_i\log\frac{\|\X^\star\|_*}{\beta} + \mu_i\log\frac{\mu_i + \sqrt{\mu_i^2 + \beta^2}}{\|\X^\star\|_*} - \sqrt{\mu_i^2 + \beta^2}.
    \end{align*}
    For any $x, \beta > 0$, we have
    \begin{equation*}
      x \le \sqrt{x^2 + \beta^2} \le x + \beta.
    \end{equation*}
    Rearranging above inequality for the nuclear norm $\|\X_\infty\|_*$, we obtain the upper bound 
  \begin{align*}
    \|\X_\infty\|_* &\le \|\X^\star\|_* + \frac{1}{\log\frac{\|\X^\star\|_*}{\beta} - 1}\sum_{i=1}^{n}\mu_i\log\frac{\mu_i + \sqrt{\mu_i^2 + \beta^2}}{\|\X^\star\|_*} - \sigma_i\log\frac{\sigma_i + \sqrt{\sigma_i^2 + \beta^2}}{\|\X^\star\|_*} + \beta\\
    &\le \|\X^\star\|_* + \frac{1}{\log\frac{\|\X^\star\|_*}{\beta} - 1} \biggl(\|\X_\infty\|_*\log 2.1 - \|\X_\infty\|_*\log\frac{2\|\X_\infty\|_*}{n\|\X^\star\|_*} + n\beta\biggr) \\
    &\le \|\X^\star\|_* + \frac{1}{\log\frac{\|\X^\star\|_*}{\beta} - 1}\Biggl(\|\X_\infty\|_*\log (1.05n) + n\beta\Bigr),
  \end{align*}
  where we used the assumptions $\beta \le \frac{\|\X^\star\|_*}{1.05en}$ and $\|\X^\star\|_*< \|\X_\infty\|_*$, and for the second inequality we used the fact that the constrained optimization problem
  \begin{equation*}
    \operatorname{optimize} \sum_{i=1}^nx_i\log\biggl(x_i + \sqrt{x_i^2 + \beta^2}\biggr) \qquad \text{s.t. } \sum_{i=1}^n x_i = K, \quad x_i \ge 0 \text{ for all } i = 1,\dots, n
  \end{equation*}
   attains a maximum when $x_i = K$ for exactly one $i\in\{1,\dots,n\}$, and attains a minimum when all $x_i = K/n$ are equal. Hence, again using the assumption $\beta < \frac{\|\X^\star\|_*}{1.05en}$, we can bound 
    \begin{equation}\label{eq:bound_nucnorm}
      \|\X_\infty\|_* \le (1 + \Delta_\beta)\Biggl(\|\X^\star\|_* + \frac{n\beta}{\log\frac{\|\X^\star\|_*}{\beta} - 1}\Biggr),
    \end{equation}
    where $\Delta_\beta = (\frac{\log(\|\X^\star\|_*/\beta) - 1}{\log(1.05n)} - 1)^{-1} > 0$, since $\beta \le \frac{\|\X^\star\|_*}{1.05en}$.

    \textbf{Step 2: Bound the reconstruction error $\|\X_\infty - \X^\star\|_F$.}\\
    With this, we can now proceed as in \cite{FCRP08,RFP10}. Writing $\mathbf{R} = \X_\infty - \X^\star$, we can apply Lemma 3.4 from \cite{RFP10} to the matrices $\X^\star$ and $\mathbf{R}$ to decompose $\mathbf{R} = \mathbf{R}_0 + \mathbf{R}_c$, where $\operatorname{rank}(\mathbf{R}_0)\le 2\operatorname{rank}(\X^\star)$, $\X^\star\mathbf{R}_c^\top = \mathbf{0}$ and $(\X^\star)^\top\mathbf{R}_c = \mathbf{0}$. We can bound
    \begin{align*}
      \|\X^\star + \mathbf{R}\|_* \ge \|\X^\star + \mathbf{R}_c\|_* - \|\mathbf{R}_0\|_* = \|\X^\star\|_* + \|\mathbf{R}_c\|_* - \|\mathbf{R}_0\|_*,
    \end{align*}
    where the inequality follows from the triangle inequality, and the equality holds since $\X^\star\mathbf{R}_c^\top = \mathbf{0}$ and $(\X^\star)^\top\mathbf{R}^c = \mathbf{0}$ together imply that the nuclear norm decomposes, see e.g.\ Lemma 2.3 of \cite{RFP10}. Together with \eqref{eq:bound_nucnorm}, this implies
    \begin{equation*}
      \|\mathbf{R}_c\|_* \le \|\mathbf{R}_0\|_* + \Delta_\beta\|\X^\star\|_* + (1+\Delta_\beta)\frac{n\beta}{\log\frac{\|\X^\star\|_*}{\beta} - 1}.
    \end{equation*}
    Next, we partition $\mathbf{R}_c$ into a sum of matrices $\mathbf{R}_1,\mathbf{R}_2,\dots, \mathbf{R}_{\lceil \frac{n}{3r}\rceil}$, with each being of rank at most $3r$. Letting $\mathbf{R}_c = \U\boldsymbol{\Sigma}\V^\top$ be the singular value decomposition of $\mathbf{R}_c$, where the diagonal elements of $\Sigma$ are in non-increasing order $\sigma_1\ge\sigma_2\ge\dots\ge\sigma_n\ge 0$, define $\mathbf{R}_i = \U_{I_i}\boldsymbol{\Sigma}_{I_i}\V_{I_i}^\top$, where $I_i = \{3r(i-1) + 1,\dots,3ri\}$. By construction, we have
    \begin{equation*}
      \sigma_k \le \frac{1}{3r}\sum_{j\in I_i}\sigma_j \qquad \text{for all } k\in I_{i+1}, \; i \in \Bigl\{1,\dots,\Bigl\lceil\frac{n}{3r}\Bigr\rceil\Bigr\},
    \end{equation*}
    which implies $\|\mathbf{R}_{i+1}\|_F^2\le \frac{1}{3r}\|\mathbf{R}_i\|_*^2$. With this, we can bound
    \begin{align}
      \sum_{j\ge 2}\|\mathbf{R}_j\|_F &\le \frac{1}{\sqrt{3r}}\sum_{j\ge 1}\|\mathbf{R}_j\|_* \nonumber\\
      &= \frac{1}{\sqrt{3r}}\|\mathbf{R}_c\|_* \nonumber\\
      &\le \frac{\sqrt{2r}}{\sqrt{3r}}\|\mathbf{R}_0\|_F + \frac{\Delta_\beta\|\X^\star\|_* + (1 + \Delta_\beta)\frac{n\beta}{\log\frac{\|\X^\star\|_*}{\beta} - 1}}{\sqrt{3r}}, \label{eq:bound_rjsum}
    \end{align}
    where for the last inequality we used that $\operatorname{rank}(\mathbf{R}_0) \le 2r$.
    Since $\mathbf{R}_0 + \mathbf{R}_1$ is at most of rank $5r$, we can use the triangle inequality and the restricted isometry property to bound
    \begin{align}
      \biggl(\frac{1}{m}\sum_{i=1}^m\langle \A_i,\mathbf{R} \rangle^2\biggr)^{1/2} &\ge \biggl(\frac{1}{m}\sum_{i=1}^m\langle \A_i,\mathbf{R}_0 + \mathbf{R}_1 \rangle^2\biggr)^{1/2} - \sum_{j\ge 2}\biggl(\frac{1}{m}\sum_{i=1}^m\langle \A_i,\mathbf{R}_j \rangle^2\biggr)^{1/2} \nonumber\\
      &\ge (1-\delta) \|\mathbf{R}_0 + \mathbf{R}_1\|_F - \sum_{j\ge 2}(1+\delta)\|\mathbf{R}_j\|_F. \label{eq:bound_rip}
    \end{align}
    Since $\mathbf{R}_0$ is orthogonal to $\mathbf{R}_1$ (see Lemma 3.4 in \cite{RFP10}), we have $\|\mathbf{R}_0+\mathbf{R}_1\|_F \ge \|\mathbf{R}_0\|_F$. By definition, we have $f(\X^\star) = f(\X_\infty) = 0$, which implies $\langle\A_i,\mathbf{R}\rangle = 0$ for all $i=1,\dots,m$. Hence, we can use \eqref{eq:bound_rjsum} and rearrange \eqref{eq:bound_rip} for $\|\mathbf{R}_0 + \mathbf{R}_1\|_F$ to obtain
    \begin{equation*}
      \|\mathbf{R}_0 + \mathbf{R}_1\|_F \le \biggl(1 - \sqrt{\frac{2}{3}} - \delta\biggl(1 + \sqrt{\frac{2}{3}}\biggr)\biggr)^{-1}\bigl(1 + \delta\bigr)\frac{\Delta_\beta\|\X^\star\|_* + (1 + \Delta_\beta)\frac{n\beta}{\log\frac{\|\X^\star\|_*}{\beta} - 1}}{\sqrt{3r}}.
    \end{equation*}
    Finally, this yields
    \begin{align*}
      \|\mathbf{R}\|_F &\le \|\mathbf{R}_0 + \mathbf{R}_1\|_F + \sum_{j\ge 2}\|\mathbf{R}_j\|_F \\
      &\le 2\biggl(1 - \sqrt{\frac{2}{3}} - \delta\biggl(1 + \sqrt{\frac{2}{3}}\biggr)\biggr)^{-1}\frac{\Delta_\beta\|\X^\star\|_* + (1 + \Delta_\beta)\frac{n\beta}{\log\frac{\|\X^\star\|_*}{\beta} - 1}}{\sqrt{3r}},
    \end{align*}
    which completes the proof of the bound \eqref{eq:sensing1} in Theorem \ref{theorem:sensing}. The bound \eqref{eq:sensing2} can be shown following the same steps, and we omit the details to avoid repetition. 
  \end{proof}

\subsection{Proof of Theorem \ref{theorem:completion}}
\begin{proof}[Proof of Theorem \ref{theorem:completion}]
As in Theorem \ref{theorem:sensing}, we first consider the rectangular case and show the bound \eqref{eq:completion1} in Theorem \ref{theorem:completion}. The proof of Theorem \ref{theorem:completion} combines the ideas from and closely follows the proofs of Theorem 2 in \cite{R11} and Theorem 7 in \cite{CP10}. It was shown in Proposition 3 in \cite{R11} that it suffices to consider a setting where the entries are sampled independently and uniformly with replacement. We first introduce some notation necessary for the proof. A more detailed background on the following quantities can be found in \cite{R11}.

We use calligraphic letters to denote linear operators on matrices, for instance, we denote the identity operator by $\mathcal{I}$. We define the spectral norm of an operator as $\|\mathcal{A}\| = \sup_{\X: \|\X\|_F\le 1}\|\mathcal{A}(\X)\|_F$. Let $\Omega = \{(a_i,b_i)\}_{i=1}^m$ be a collection of indices sampled uniformly at random with replacement (possibly containing repetitions), and define the operator
\begin{equation*}
  \mathcal{R}_\Omega(\X) = \sum_{i=1}^m\langle \mathbf{e}_{a_i}\mathbf{e}_{b_i}^\top, \X\rangle \mathbf{e}_{a_i}\mathbf{e}_{b_i}^\top.
\end{equation*}
Let $\X^\star = \U\boldsymbol{\Sigma}\V^\top$ be the singular value decomposition of $\X^\star$, and let $\mathbf{u}_k$ (resp. $\mathbf{v}_k$) be the $k$-th column of $\U$ (resp. $\V$), and define the subspaces $U = \operatorname{span}(\mathbf{u}_1,\dots,\mathbf{u}_r)$ and $V = \operatorname{span}(\mathbf{v}_1,\dots,\mathbf{v}_r)$. Let $T$ be the linear space spanned by elements of the form $\mathbf{u}_k\mathbf{y}^\top$ and $\mathbf{x}\mathbf{v}_k^\top$, $k=1,\dots,r$, where $\mathbf{x}\in\R^{n}$ and $\mathbf{y}\in\R^{n'}$ are arbitrary vectors, and let $T^\bot$ be its orthogonal complement. The orthogonal projection onto the subspace $T$ is given by
\begin{equation*}
  \mathcal{P}_T(\X) = \mathbf{P}_U\X + \X\mathbf{P}_V - \mathbf{P}_U\X\mathbf{P}_V,
\end{equation*}
where $\mathbf{P}_U$ and $\mathbf{P}_V$ are the orthogonal projections onto $U$ and $V$, respectively. Then, the orthogonal projection onto $T^\bot$ is given by 
\begin{equation*}
  \mathcal{P}_{T^\bot}(\X) = (\mathcal{I} - \mathcal{P}_T)(\X).
\end{equation*}
It has been shown in \cite{R11} that, with high probability, 
\begin{equation}\label{eq:completion_proof1}
\frac{nn'}{m}\biggl\|\mathcal{P}_T\mathcal{R}_\Omega\mathcal{P}_T - \frac{m}{nn'}\mathcal{P}_T\biggr\|\le \frac{1}{2}, \qquad \|\mathcal{R}_\Omega\| \le \frac{8}{3}\sqrt{c}\log n', 
\end{equation}
and that there exists a matrix $\Y$ in the range of $\mathcal{R}_\Omega$ satisfying
\begin{equation}\label{eq:completion_proof2}
\|\mathcal{P}_T(\Y) - \U\V^\top\|_F \le \sqrt{\frac{r}{2n'}}, \qquad \|\mathcal{P}_{T^\bot}(\Y)\|_F \le \frac{1}{2},
\end{equation}
see Section 4 in \cite{R11} for a proof of these statements. 

Let $\mathbf{R} = \X_\infty - \X^\star$. Since the subspaces $T$ and $T^\bot$ are orthogonal by construction, we have
\begin{equation*}
\|\mathbf{R}\|_F^2 = \|\mathcal{P}_T(\mathbf{R})\|_F^2 + \|\mathcal{P}_{T^\bot}(\mathbf{R})\|_F^2,
\end{equation*}
so the goal is to bound the two terms on the right hand side of above identity. Since both $\X^\star$ and $\X_\infty$ are global minimizers of the empirical risk $f$, we have
\begin{equation*}
0 = \|\mathcal{R}_\Omega(\mathbf{R})\|_F \ge \|\mathcal{R}_\Omega\mathcal{P}_T(\mathbf{R})\|_F - \|\mathcal{R}_\Omega\mathcal{P}_{T^\bot}(\mathbf{R})\|_F,
\end{equation*}
where we used the reverse triangle inequality. Further, the first bound in \eqref{eq:completion_proof1} implies
\begin{equation*}
\|\mathcal{R}_\Omega\mathcal{P}_T(\mathbf{R})\|_F^2 = \langle\mathbf{R}, \mathcal{P}_T\mathcal{R}_\Omega^2\mathcal{P}_T(\mathbf{R}) \rangle \ge \langle\mathbf{R}, \mathcal{P}_T\mathcal{R}_\Omega\mathcal{P}_T(\mathbf{R}) \rangle \ge \frac{m}{2nn'}\|\mathcal{P}_T(\mathbf{R})\|_F^2,
\end{equation*}
and, using the second bound in \eqref{eq:completion_proof1}, we can bound $\|\mathcal{R}_\Omega\mathcal{P}_{T^\bot}(\mathbf{R})\|_F\le \frac{8}{3}\sqrt{c}\log(n')\|\mathcal{P}_{T^\bot}(\mathbf{R})\|_F$. Together, this implies
\begin{equation}\label{eq:completion_proof3}
\|\mathcal{P}_{T^\bot}(\mathbf{R})\|_F \ge \sqrt{\frac{9m}{128cnn'\log^2n'}}\|\mathcal{P}_T(\mathbf{R})\|_F \ge \sqrt{\frac{4.5r}{n'}}\|\mathcal{P}_T(\mathbf{R})\|_F.
\end{equation}
Recalling the variational characterization of the nuclear norm $\|\A\|_* = \sup_{\mathbf{B}:\|\mathbf{B}\|\le 1}\langle\A,\mathbf{B}\rangle$, we can choose matrices $\U_\bot$ and $\V_\bot$ such that $[\U, \U_\bot]$ and $[\V, \V_\bot]$ are orthogonal matrices and $\langle\U_\bot\V_\bot^\top, \mathcal{P}_{T^\bot}(\mathbf{R})\rangle = \|\mathcal{P}_{T^\bot}(\mathbf{R})\|_*$. Let $\Y$ be as in \eqref{eq:completion_proof2}. Then, we can bound
\begin{align*}
\|\X^\star + \mathbf{R}\|_* &\ge \langle \U\V^\top + \U_\bot\V_\bot^\top, \X^\star + \mathbf{R} \rangle \\
&= \|\X^\star\|_* + \langle \U\V^\top + \U_\bot\V_\bot^\top, \mathbf{R} \rangle \\
&= \|\X^\star\|_* + \langle \U\V^\top + \U_\bot\V_\bot^\top - (\mathcal{P}_T(\Y) + \mathcal{P}_{T^\bot}(\Y)), \mathcal{P}_T(\mathbf{R}) + \mathcal{P}_{T^\bot}(\mathbf{R}) \rangle \\
&= \|\X^\star\|_* + \langle \U\V^\top - \mathcal{P}_T(\Y), \mathcal{P}_T(\mathbf{R}) \rangle + \langle \U_\bot\V_\bot^\top - \mathcal{P}_{T^\bot}(\Y), \mathcal{P}_{T^\bot}(\mathbf{R})  \rangle \\
&\ge \|\X^\star\|_* - \sqrt{\frac{r}{2n'}}\|\mathcal{P}_T(\mathbf{R})\|_F + \frac{1}{2}\|\mathcal{P}_{T^\bot}(\mathbf{R})\|_* \\
&\ge \|\X^\star\|_* + \frac{1}{6}\|\mathcal{P}_{T^\bot}(\mathbf{R})\|_*,
\end{align*}
where the first line follows from the varitional characterization of the nuclear norm, the third line from the fact that $\Y$ and $\mathbf{R}$ are orthogonal since $\Y$ is in the range and $\mathbf{R}$ in the kernel of $\mathcal{R}_\Omega$, the fourth line from the fact that $T$ and $T^\bot$ are, by construction, orthogonal subspaces, the fifth line from the bound \eqref{eq:completion_proof2} and the definition of $\U_\bot, \V_\bot$, and the last line follows from the bound \eqref{eq:completion_proof3} and the fact that the Frobenius norm is bounded from above by the nuclear norm. 

As in the proof of Theorem \ref{theorem:sensing}, we can bound the nuclear norm
\begin{equation*}
  \|\X_\infty\|_* \le (1 + \Delta_\beta)\biggl(\|\X^\star\|_* + \frac{n\beta}{\log\frac{\|\X^\star\|_*}{\beta} - 1}\biggr),
\end{equation*}
where $\Delta_\beta = (\frac{\log(\|\X^\star\|_*/\beta) - 1}{\log(1.05n)} - 1)^{-1}$.
Hence, we can bound
\begin{equation*}
  \|\mathcal{P}_{T^\bot}(\mathbf{R})\|_F \le \|\mathcal{P}_{T^\bot}(\mathbf{R})\|_* \le 6\biggl(\Delta_\beta\|\X^\star\|_* + (1+\Delta_\beta)\frac{n\beta}{\log\frac{\|\X^\star\|_*}{\beta} - 1}\biggr).
\end{equation*}
Using \eqref{eq:completion_proof3}, we can also bound
\begin{align*}
  \|\mathcal{P}_{T}(\mathbf{R})\|_F &\le \sqrt{\frac{128cnn'\log^2n'}{9m}}\|\mathcal{P}_{T^\bot}(\mathbf{R})\|_F \\
  &\le 6\biggl(\Delta_\beta\|\X^\star\|_* + (1+\Delta_\beta)\frac{n\beta}{\log\frac{\|\X^\star\|_*}{\beta} - 1}\biggr)\sqrt{\frac{128cnn'\log^2n'}{9m}}.
\end{align*}
Putting everything together, we have
\begin{equation*}
\|\mathbf{R}\|_F \le 6\biggl(\Delta_\beta\|\X^\star\|_* + (1+\Delta_\beta)\frac{n\beta}{\log\frac{\|\X^\star\|_*}{\beta} - 1}\biggr)\biggl(1 + \biggl(\frac{128cnn'\log^2n'}{9m}\biggr)^{\frac{1}{2}}\biggr),
\end{equation*}
which completes the proof of the bound \eqref{eq:completion1} in Theorem \ref{theorem:completion}. The bound \eqref{eq:completion2} can be shown following the same steps, and we omit the details to avoid repetition. 
\end{proof}

\subsection{Proof of Proposition \ref{prop1}}
\label{appendix:proofs_prop1}
\begin{proof} We begin by showing the first part of Proposition \ref{prop1}.

  \textbf{Proof of part 1.}\\
  Recalling the expressions for the gradient of the spectral entropy mirror map $\nabla\Phi$ and its inverse $\nabla\Phi^{-1}$ provided in Appendix \ref{appendix:mirror_maps_gradients}, the mirror descent update \eqref{eq:mirror_descent} becomes
  \begin{equation*}
\X_{t+1} = \exp \bigl(\log \X_t - \eta\nabla f(\X_t)\bigr).
  \end{equation*}
  By assumption, $\X_0$ commutes with all sensing matrices $\A_i$'s, and hence also with the gradient 
  \begin{equation*}
\nabla f(\X_0) = \frac{1}{m}\sum_{i=1}^m(\langle \A_i, \X_0\rangle - y_i)\A_i,
  \end{equation*}
  which is a linear combination of the $\A_i$'s. Further, note that if two matrices $\A$ and $\mathbf{B}$ commute, then the matrices $\log \A$ and $\exp(\A)$ also commute with $\mathbf{B}$. By induction, this implies that $\log \X_t$ and $\nabla f(\X_t)$ commute for all $t\ge 0$, and we therefore have
  \begin{equation*}
    \exp\bigl(\log\X_t - \eta\nabla f(\X_t)\bigr) = \X_t \exp\bigl(-\eta\nabla f(\X_t)\bigr) = \exp\bigl(-\eta\nabla f(\X_t)\bigr) \X_t,
  \end{equation*}
  where we used the fact that $e^{\A + \mathbf{B}} = e^{\A}e^{\mathbf{B}}$ if the matrices $\A$ and $\mathbf{B}$ commute. Hence, the mirror descent update \eqref{eq:mirror_descent} is equivalent to 
  \begin{equation*}
\X_{t+1} = \frac{1}{2}\Bigl(\X_t \exp\bigl(-\eta\nabla f(\X_t)\bigr) + \exp\bigl(-\eta\nabla f(\X_t)\bigr) \X_t\Bigr),
  \end{equation*}
  which is exactly the exponentiated gradient algorithm defined in \eqref{eq:exponentiated_gradient} with initialization $\U_0 = \X_0$ and $\V_0 = \mathbf{0}$.

  \textbf{Proof of part 2.}\\
  We begin by studying the mirror descent update \eqref{eq:mirror_descent}. Recalling the expressions for the gradient of the spectral hypentropy mirror map $\nabla\Phi_\beta$ and its inverse $\nabla\Phi_\beta^{-1}$ for symmetric matrices we derived in Appendix \ref{appendix:mirror_maps_gradients}, the mirror descent update \eqref{eq:mirror_descent} becomes
  \begin{align*}
\X_{t+1} &= \nabla\Phi_\beta^{-1}\bigl(\nabla\Phi_\beta(\X_t) - \eta\nabla f(\X_t)\bigr) \\
&= \frac{\beta}{2}\Bigg[ \exp\Biggl(\log \Biggl(\frac{\X_t}{\beta} + \sqrt{\frac{\X_t^2}{\beta^2} + \mathbf{I}}\Biggr) - \eta\nabla f(\X_t)\Biggr) \\
&\qquad - \exp\Biggl(-\log \Biggl(\frac{\X_t}{\beta} + \sqrt{\frac{\X_t^2}{\beta^2} + \mathbf{I}}\Biggr) + \eta\nabla f(\X_t)\Biggr) \Bigg].
  \end{align*}
  Assuming that $\X_t$ is symmetric, we can write $\X_t = \mathbf{B}\mathbf{D}\mathbf{B}^\top$ by the spectral theorem, where $\mathbf{B}$ is an orthogonal matrix and $\mathbf{D}$ a diagonal matrix. Then, we have
  \begin{align*}
    -\log \Biggl(\frac{\X_t}{\beta} + \sqrt{\frac{\X_t^2}{\beta^2} + \mathbf{I}}\Biggr) &= -\mathbf{B}\log \Biggl(\frac{\mathbf{D}}{\beta} + \sqrt{\frac{\mathbf{D}^2}{\beta^2} + \mathbf{I}}\Biggr)\mathbf{B}^\top \\
    &= \mathbf{B}\log \Biggl(-\frac{\mathbf{D}}{\beta} + \sqrt{\frac{\mathbf{D}^2}{\beta^2} + \mathbf{I}}\Biggr)\mathbf{B}^\top \\
    &= \log \Biggl(-\frac{\X_t}{\beta} + \sqrt{\frac{\X_t^2}{\beta^2} + \mathbf{I}}\Biggr),
  \end{align*}
  since we have $(x + \sqrt{x^2 + 1})^{-1} = -x + \sqrt{x^2 + 1}$ for all $x\in\R$. Assuming that $\X_t$ (and hence also $\log (\X_t/\beta + \sqrt{(\X_t/\beta)^2 + \mathbf{I}})$) commutes with all $\A_i$'s, the mirror descent update can be written as
  \begin{align*}
\X_{t+1} &= \frac{\beta}{2}\Biggl[\Biggl(\frac{\X^2}{\beta} + \sqrt{\frac{\X^2}{\beta^2} + \mathbf{I}}\Biggr) \exp\bigl(-\eta\nabla f(\X_t)\bigr) - \Biggl(-\frac{\X^2}{\beta} + \sqrt{\frac{\X^2}{\beta^2} + \mathbf{I}}\Biggr) \exp\bigl(\eta\nabla f(\X_t)\bigr)\Biggr] \\
&= \frac{1}{2}\Bigg[\exp\Bigl(-\eta\nabla f(\X_t)\Bigr) \frac{\X_t + \sqrt{\X_t^2 + \beta^2\mathbf{I}}}{2} +  \frac{\X_t + \sqrt{\X_t^2 + \beta^2\mathbf{I}}}{2} \exp\Bigl(-\eta\nabla f(\X_t)\Bigr) \\
&\hspace{2.5em} + \exp\Bigl(\eta\nabla f(\X_t)\Bigr) \frac{-\X_t + \sqrt{\X_t^2 + \beta^2\mathbf{I}}}{2} + \frac{-\X_t + \sqrt{\X_t^2 + \beta^2\mathbf{I}}}{2} \exp\Bigl(\eta\nabla f(\X_t)\Bigr) \Bigg].
  \end{align*}
  Since $\X_0 = \mathbf{0}$ is symmetric and commutes with all $\A_i$'s, this identity inductively shows that $\X_t$ is symmetric and commutes with all $\A_i$'s for all $t\ge 0$.

  Next, consider the exponentiated gradient algorithm \eqref{eq:exponentiated_gradient} with initialization $\U_0 = \V_0 = \frac{1}{2}\beta\mathbf{I}$. Since the initializations $\U_0$ and $\V_0$ both commute with all $\A_i$'s, the update \eqref{eq:exponentiated_gradient} implies that $\U_t$ and $\V_t$ commute with all $\A_i$'s for all $t\ge 0$. Then, we have
  \begin{align*}
    \U_{t+1}\V_{t+1} = \frac{e^{-\eta\nabla f(\X_t)}\U_t + \U_t e^{-\eta\nabla f(\X_t)}}{2} \cdot \frac{e^{\eta\nabla f(\X_t)}\V_t + \V_te^{\eta\nabla f(\X_t)}}{2} = \U_t\V_t,
  \end{align*}
  that is the product $\U_t\V_t = \U_0\V_0 = \frac{1}{4}\beta^2\mathbf{I}$ stays constant for all $t\ge 0$. Since the matrix exponential of a symmetric matrix is always positive definite, the update \eqref{eq:exponentiated_gradient} also implies that $\U_t$ and $\V_t$ stay positive definite for all $t\ge 0$, so that $\U_t$ and $\V_t$ are invertible. Together with the definition $\X_t = \U_t - \V_t$, we can solve for
  \begin{equation*}
\U_t = \frac{\X_t + \sqrt{\X_t^2 + \beta^2\mathbf{I}}}{2}, \qquad \V_t = \frac{-\X_t + \sqrt{\X_t^2 + \beta^2\mathbf{I}}}{2},
  \end{equation*}
  which completes the proof that mirror descent \eqref{eq:mirror_descent} is equivalent to the exponentiated gradient algorithm \eqref{eq:exponentiated_gradient} when the sensing matrices $\A_i$'s are symmetric and commute. 
\end{proof}

\subsection{Further claims}
\label{appendix:proofs_further_claims}
In this section, we elaborate on and justify further claims made in the main paper.

First, we demonstrate that minimizing the quantity 
\begin{equation}\label{eq:claim1_helper}
  \sum_{i=1}^{n}\sigma_i\log\frac{1}{\beta} + \sigma_i\log\Bigl(\sigma_i + \sqrt{\sigma_i^2 + \beta^2}\Bigr) - \sqrt{\sigma_i^2 + \beta^2}
\end{equation}
corresponds to minimizing the nuclear norm in the limit $\beta\rightarrow 0$ and to minimizing the Frobenius norm in the limit $\beta\rightarrow \infty$, see also \cite{WGL+20}, which showed the analogous result in the vector-case.

In the limit $\beta\rightarrow 0$, the term $\log\frac{1}{\beta}$ converges to infinity, hence minimizing the quantity in \eqref{eq:claim1_helper} corresponds to minimizing the nuclear norm $\sum_{i=1}^n\sigma_i$.

In the limit $\beta\rightarrow \infty$, we can substitute $z_i = \sigma_i/\beta$ and write the expression in \eqref{eq:claim1_helper} as
\begin{equation}
  \beta \sum_{i=1}^{n}z_i\log\Bigl(z_i + \sqrt{z_i^2 + 1}\Bigr) - \sqrt{z_i^2 + 1} = \beta \sum_{i=1}^n-1 + \frac{z_i^2}{2} + \O\bigl(z_i^2\bigr),
\end{equation}
where we applied a Taylor expansion around $z_i = 0$. Hence, minimizing the quantity in \eqref{eq:claim1_helper} corresponds to minimizing the Frobenius norm $(\sum_{i=1}^n\sigma_i^2)^{1/2}$ in the limit $\beta\rightarrow\infty$.

Next, we demonstrate that gradient descent with full-rank parametrization $\X = \U\U^\top - \V\V^\top$, where $\U,\V\in\R^{n\times n}$, is a first order approximation to the exponentiated gradient algorithm defined in \eqref{eq:exponentiated_gradient}, with the step size rescaled by a factor $4$ and the approximation being exact in the limit $\eta\rightarrow 0$, i.e.\ the continuous-time algorithms are equivalent.  

First, using the first-order approximation $e^{\eta\A} = \mathbf{I} + \eta\A + \O(\eta^2)$, the exponentiated gradient algorithm becomes
\begin{equation*}
  \begin{gathered}
  \X_t = \U_t - \V_t \\
  \U_{t+1} \approx \U_t - \eta \frac{\U_t\nabla f(\X_t) + \nabla f(\X_t)\U_t}{2}, \qquad \V_{t+1} \approx \V_t + \eta\frac{\V_t\nabla f(\X_t) + \nabla f(\X_t)\V_t}{2},
  \end{gathered}
\end{equation*}
where we omitted higher order $\O(\eta^2)$ terms. On the other hand, the update for gradient descent is given by
\begin{equation*}
  \begin{gathered}
  \X_t = \U_t\U_t^\top - \V_t\V_t^\top \\
  \U_{t+1} = \U_t - 2\eta\nabla f(\X_t), \qquad \V_{t+1} = \V_t + 2\eta\nabla f(\X_t).
  \end{gathered}
\end{equation*}
With this, we can compute $\U_{t+1}\U_{t+1}^\top = \U_t + 2\eta (\U_t\nabla f(\X_t) + \nabla f(\X_t)\U_t) + \O(\eta^2)$, so gradient descent with full-rank parametrization $\X = \U\U^\top - \V\V^\top$, $\U,\V\in\R^{n\times n}$, is indeed a first-order approximation of the exponentiated gradient algorithm defined in \eqref{eq:exponentiated_gradient}, with the step size rescaled by a factor $4$. Hence, in the limit $\eta\rightarrow 0$, the differentials $\frac{d\X_t}{dt} = \lim_{\eta\rightarrow 0}\frac{\X_{t+\eta} - \X_t}{\eta}$ of the exponentiated gradient algorithm \eqref{eq:exponentiated_gradient} and gradient descent with full-rank factorized parametrization coincide.

\section{Additional experiments for matrix completion}
\label{appendix:additional_experiments}
In this section, we present additional numerical simulations which consider matrix completion, i.e.\ the sensing matrices $\A_i$'s each have exactly one random entry set to one and all other entries set to zero. The remaining exeperimental setup is as described in Section \ref{sec:numerical_simulations}, with the difference that we choose step sizes $\mu = 2000$ and $\mu = 500$ for mirror descent and gradient descent, respectively, due to the lower spectral norm of the sensing matrices $\A_i$'s in matrix completion compared to matrix sensing with random Gaussian sensing matrices. As the experiments for Figure \ref{figure1}, the experiments for Figure \ref{figure2} were implemented in Python 3.9 and took around $10$ minutes on a machine with 1.1-GHz Intel Core i5 CPU and 8 GB of RAM.
\begin{figure}[!htb]
  \centering
  \includegraphics[width=0.925\textwidth]{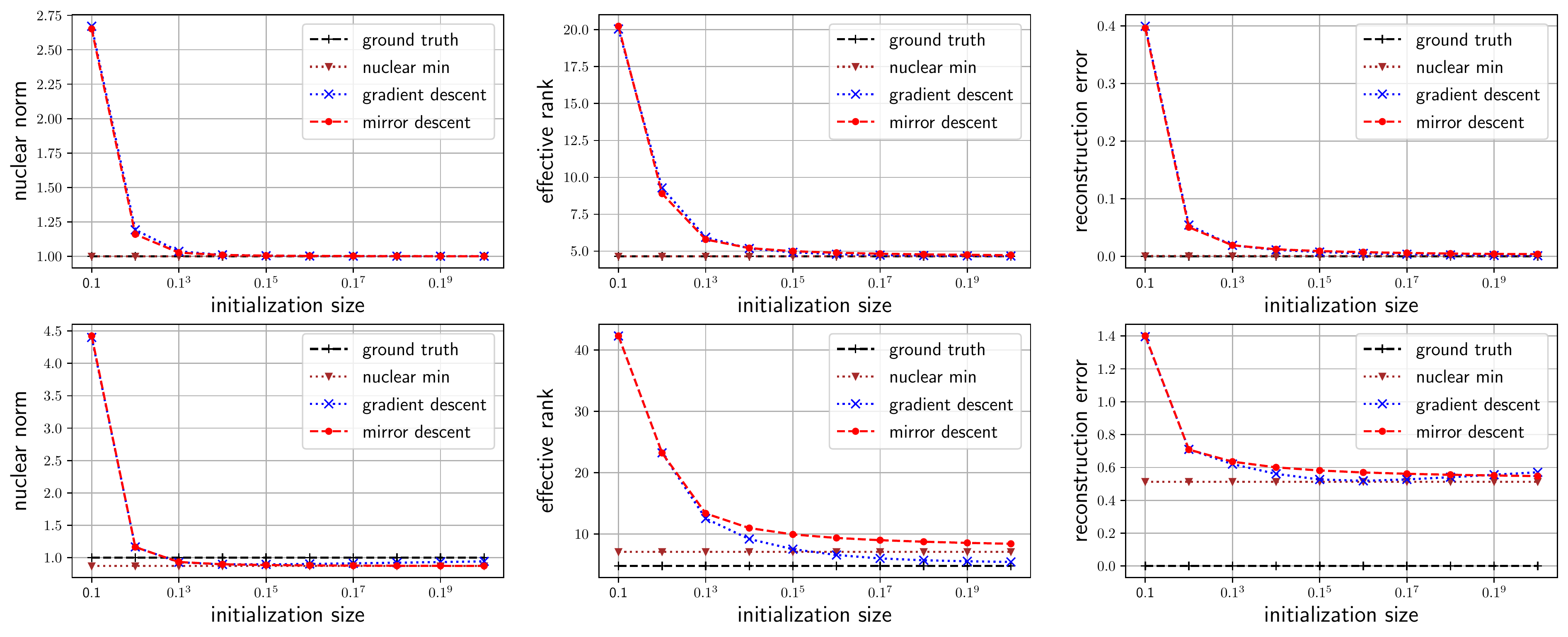}
  \caption{Nuclear norm, effective rank \cite{RV07} and reconstruction error in matrix completion against initialization size $\alpha$ for $n=50$ and $r=5$. Top row: $m=3nr$. Bottom row: $m=nr$.}\vspace{-0.5em}
  \label{figure2}
  \end{figure}

We consider the nuclear norm $\|\X\|_*$, the effective rank defined in \cite{RV07} and the reconstruction error $\|\X - \X^\star\|_F$ of the estimates from mirror descent, gradient descent and nuclear norm minimization, and compare these quantities to the ground truth $\X^\star$. Figure \ref{figure2} shows that the results in matrix completion qualitatively match the results in Figure \ref{figure1} for matrix sensing with random Gaussian sensing matrices. In particular, with $m=3nr$ observed entries (Figure \ref{figure2}, top row), nuclear norm minimization recovers the planted matrix $\X^\star$ and the estimates of mirror descent and gradient descent closely track each other in terms of the quantities considered. When only $m=nr$ entries are observed (Figure \ref{figure2}, bottom row), nuclear norm minimization does not recover the planted matrix $\X^\star$, and we observe that gradient descent puts more emphasis on lowering the effective rank at the expense of a (slightly) higher nuclear norm for initialization sizes smaller than $10^{-3}$.

\end{document}